\documentclass[11pt]{article}

\usepackage{jmlr2e}
\usepackage{times,tikz}
\usepackage{graphicx}
\usetikzlibrary{matrix,arrows,decorations.pathmorphing,calc,shadings,decorations.pathreplacing,patterns}

\title{Optimal Collusion-Free Teaching\thanks{This is an extended version of~\citep{KSZ2019}}}

 \author{\name David Kirkpatrick \email kirk@cs.ubc.ca \\
 \addr Department of Computer Science, University of British Columbia
 \\
 \AND
 \name Hans U. Simon \email hans.simon@rub.de\\
 \addr Department of Mathematics, Ruhr-University Bochum
 \\
 \AND
 \name Sandra Zilles \email zilles@cs.uregina.ca\\
 \addr Department of Computer Science, University of Regina
 }

\usepackage{bbm}

\newcommand{\PBTD}{\mathrm{PBTD}}

\newcommand{\TD}{\mathrm{TD}}

\newcommand{\RTD}{\mathrm{RTD}}

\newcommand{\dom}{\mathrm{dom}}

\newcommand{\cX}{{\mathcal{X}}}
\newcommand{\cS}{{\mathcal{S}}}
\newcommand{\cC}{{\mathcal{C}}}
\newcommand{\cP}{{\mathcal{P}}}

\newcommand{\seq}{\subseteq}

\newcommand{\eset}{\emptyset}
\newcommand{\VCD}{\mbox{$\mathrm{VCD}$}}

\newcommand{\NCTD}{\mathrm{NCTD}}
\newcommand{\ANCTD}{\mathrm{ANCTD}}

\newcommand{\CN}{\mathrm{CN}}
\newcommand{\TC}{\mathrm{ord}}

\begin{document}



\maketitle

\begin{abstract}
Formal models of learning from teachers need to respect certain criteria to avoid collusion. The most commonly accepted notion of collusion-freeness was proposed by \cite{GoldmanM96}, and various teaching models obeying their criterion have been studied. For each model $M$ and each concept class $\cC$, a parameter $M$-$\TD(\cC)$ refers to the \emph{teaching dimension} of concept class $\cC$ in model $M$---defined to be the number of examples required for teaching a concept, in the worst case over all concepts in $\cC$.

This paper introduces a new model of teaching, called no-clash teaching, together with the corresponding parameter $\NCTD(\cC)$. No-clash teaching is provably optimal in the strong sense that, given \emph{any}\/ concept class $\cC$ and \emph{any}\/ model $M$ obeying Goldman and Mathias's collusion-freeness criterion, one obtains $\NCTD(\cC)\le M$-$\TD(\cC)$. We also study a corresponding notion $\NCTD^+$ for the case of learning from positive data only, establish useful bounds on $\NCTD$ and $\NCTD^+$, and discuss relations of these parameters to the VC-dimension and to sample compression.

In addition to formulating an optimal model of collusion-free teaching, 
our main results are on the computational complexity of deciding whether $\NCTD^+(\cC)=k$ (or $\NCTD(\cC)=k$) for given $\cC$ and $k$. We show some such decision problems to be equivalent to 
the existence question for certain constrained matchings in bipartite
graphs. Our NP-hardness results for the latter are of independent interest in the study of constrained graph matchings.
\end{abstract}


\section{Introduction}

%

Models of machine learning from carefully chosen examples, i.e., from teachers, have gained increased interest in recent years, due to various application areas, such as robotics \citep{ArgallCVB09}, trustworthy AI \citep{ZhuSZR18}, and pedagogy \citep{ShaftoGG14}. Machine teaching is also related to inverse reinforcement learning \citep{HoLMCA16}, to sample compression \citep{MSWY2015,DFSZ2014}, and to curriculum learning \citep{BengioLCW09}. The paper at hand is concerned with abstract notions of teaching, as studied in computational learning theory.

A variety of formal models of teaching have been proposed in the literature, for example, the classical teaching dimension model \citep{GK1995}, the optimal teacher model \citep{B2008}, recursive teaching \citep{ZLHZ2011}, 
or preference-based teaching \citep{GaoRSZ17}. 

In each of these models, a mapping $T$ (the \emph{teacher}\/) assigns a finite set $T(C)$ of correctly labelled examples to a concept $C$ in a concept class $\mathcal{C}$ in a way that the learner can reconstruct $C$ from $T(C)$.
Intuitively, unfair collusion between the teacher and the learner should not be allowed in any formal model of teaching. 
For example, one would not want the teacher and learner to agree on a total order over the domain and a total order over the concept class and then to simply use the $i$th instance in the domain for teaching the $i$th concept, irrespective of the actual structure of the concept class.

However, there is no general definition of what constitutes  collusion, and of what constitutes desirable or undesirable forms of learning. In this manuscript, we focus on a notion of collusion that was proposed by \cite{GoldmanM96} and that has been adopted by the majority of teaching models studied in the literature. In a nutshell, Goldman and Mathias's model demands that, (i) the examples in $T(C)$ are labelled consistently with $C$, and (ii) if the learner correctly identifies $C$ from $T(C)$, then it will also identify $C$ from any superset $S$ of $T(C)$ as long as the sample set $S$ remains consistent with $C$. In other words, adding more information about $C$ to $T(C)$ will not divert the learner to an incorrect hypothesis.

Most existing abstract models of machine teaching are collusion-free in this sense. Historically, some of these models were designed in order to overcome weaknesses of the previous models. For example, the optimal teacher model by \cite{B2008} is designed to overcome limitations of the classical teaching dimension model, and was likewise superseded by the recursive teaching model \citep{ZLHZ2011}. The latter again was inapplicable to many interesting infinite concept classes, which gave rise to the model of preference-based teaching \citep{GaoRSZ17}. Each model strictly dominates the previous one in terms of the \emph{teaching complexity}\/, i.e., the worst-case number of examples needed for teaching a concept in the underlying concept class $\cC$. In this context, one quite natural question has been ignored in the literature to date: what is the smallest teaching complexity that can be achieved under Goldman and Mathias's condition of collusion-freeness? This is exactly the question addressed in this paper.

Our first contribution is the formal definition of a collusion-free teaching model that has, for every concept class $\cC$, the provably smallest teaching complexity among all collusion-free teaching models. We call this model \emph{no-clash teaching}, since its core property, which turns out to be characteristic for collusion-freeness, requires that no pair of concepts are consistent with the union of their teaching sets. A similar property was used once in the literature in the context of sample compression schemes \citep{KW2007}, and dubbed the \emph{non-clashing} property. 

For example, consider a concept class (i.e., set system) $\cC$ over the instance space $\{1,2,3,4\}$, consisting of the four concepts of the form $\{i, (i+1)\bmod 4\}$ for $1\le i\le 4$. Then no-clash teaching is possible by assigning the singleton set $\{(i,1)\}$ (interpreted as the information ``$i$ belongs to  the target concept'') as a teaching set to the concept $\{i, (i+1)\bmod 4\}$; no two distinct concepts are consistent with the union of their assigned teaching sets. Thus, in the no-clash setting, each concept in $\cC$ can be taught with a single example. By comparison, consider the classical teaching dimension model, in which a teaching set for a given concept is required to be inconsistent with all other concepts in the concept class \citep{GK1995}. It is not hard to see that, under such constraints, no concept in $\cC$ can be taught with a single example; a smallest teaching set for concept $\{i, (i+1)\bmod 4\}$ would then be $\{(i,1), ((i+1)\bmod 4,1)\}$.

We call the worst-case number of examples needed for non-clashing teaching of any concept $C$ in a given concept class $\cC$ the \emph{no-clash teaching dimension}\/ of $\cC$, abbreviated $\NCTD(\cC)$, and we study a variant $\NCTD^+(\cC)$ in which teaching uses only positive examples. In the example above, $\NCTD=\NCTD^+=1$, while the classical teaching dimension is $2$.

The value $\NCTD(\cC)$ being the smallest collusion-free teaching complexity parameter of $\cC$ makes it interesting for several reasons. 

(1) $\NCTD$ represents the limit of data efficiency in teaching when obeying Goldman and Mathias's notion of collusion-freeness. Therefore the study of $\NCTD$ has the potential to further our understanding how collusion-freeness constrains teaching. It will also help to compare other notions of collusion-freeness (see, e.g., \citep{ZLHZ2011}) to that of Goldman and Mathias.

(2) An open question in computational learning theory is whether the VC-dimension ($\VCD$), \citep{VC1971}, 
which characterizes the sample complexity of learning from randomly chosen examples, also characterizes teaching complexity for some reasonable notion of teaching. Recently, the first strong connections between teaching and $\VCD$ were established, culminating in an upper bound on the recursive teaching dimension ($\RTD$) that is quadratic in $\VCD$ \citep{HuWLW17}, but it remains open whether this bound can be improved to be linear in $\VCD$. Obviously, now $\NCTD$ is a much stronger candidate for a linear relationship with $\VCD$ than $\RTD$ is. In fact, there is no concept class known yet for which $\NCTD$ exceeds $\VCD$.

(3) The problem of relating teaching complexity to $\VCD$ is connected to the famous open problem of determining whether $\VCD$ is an upper bound on the size of the smallest possible sample compression scheme \citep{LW1986,FW1995} of a concept class. Some interesting relations between sample compression and teaching have been established for $\RTD$ \citep{MSWY2015,DFSZ2014,DarnstadtKSZ16}. The study of $\NCTD$ can potentially strengthen such relations.

In addition, an important contribution of our paper is to link $\NCTD$ to the extensively developed theory of constrained graph matching. We show that the question whether $\NCTD^+=1$ is equivalent to a very natural constrained bipartite matching problem which has apparently not yet been studied in the literature. We proceed by proving that this particular matching problem is NP-hard---a result that generalizes to larger values of $\NCTD^+$ as well as to $\NCTD$. By comparison, the question whether $\RTD^+=1$ or $\RTD=1$ can be answered in 
linear time.

To sum up, our new notion of optimal collusion-free teaching is of relevance to the study of important open problems in computational learning theory as well as of fundamental graph-theoretic decision problems, and therefore appears to be worth studying in more detail.

\section{Preliminaries}

Given a domain $\cX$, a concept over $\cX$ is a subset $C\subseteq\cX$, and we usually denote by $\cC$ a \emph{concept class}\/ over $\cX$, i.e., a set of concepts over $\cX$. 
Implicitly, we identify a concept $C$ over $\cX$ with a mapping $C:\cX\rightarrow\{0,1\}$, where $C(x)=1$ iff $x\in C$. By $\VCD(\cC)$, we denote the VC-dimension of $\cC$. 

A labelled example is a pair $(x,\ell)\in\cX\times\{0,1\}$, and it is consistent with a concept $C$ if $C(x)=\ell$. Likewise, a set $S$ of labelled examples over $\cX$, which is also called a \emph{sample set}, is consistent with $C$, if every element of $S$ is consistent with $C$. An example with the label $\ell=1$ is a positive example, while $\ell=0$ is the label of a negative example. 

Intuitively, the notion of teaching refers to compressing any concept in a given concept class to a consistent sample set.

\begin{definition}\label{def:teacher}
Let $\cC$ be a concept class over a domain $\cX$. A {\em teacher mapping} for $\cC$ is a mapping $T$ on $\cC$ such that, for all $C\in\cC$, $T(C)$ is a finite sample set $S\subseteq \cX\times\{0,1\}$ that is consistent with $C$.
\end{definition}

The first model of teaching that was proposed in the literature required from a teacher mapping $T$ that the concept $C\in\cC$ be the only concept in $\cC$ that is consistent with $T(C)$, for any $C\in\cC$ \citep{SM1991,GK1995}. This led to the definition of the well-known teaching dimension parameter. 

\begin{definition}[\cite{SM1991,GK1995}]\label{def:td}
Let $\cC$ be a concept class over a domain $\cX$ and $C\in\cC$ be a concept. A {\em teaching set} for $C$ (with respect to $\cC$) is a sample set $S$ such that $C$ is the only concept in $\cC$ consistent with $S$. The {\em teaching dimension} of $C$ in $\cC$, denoted by $\TD(C,\cC)$, is the size of the smallest teaching set for $C$ with respect to $\cC$. The teaching dimension of $\cC$ is then defined as 
$\TD(\cC)=\sup\{\TD(C,\cC)\mid C\in\cC\}$.
\end{definition}


For example, let $\cC$ be a concept class over a domain $\cX$ of exactly $m$ elements, containing the empty concept, all singleton concepts over $\cX$, and no other concepts. Then $\TD(\{x\},\cC)=1$ for each singleton concept $\{x\}$, since $\{(x,1)\}$ serves as a teaching set for $\{x\}$. By comparison, $\TD(\emptyset,\cC)=m$, since any set of up to $m-1$ negative examples is consistent with some singleton concept, so that all $m$ negative examples need to be presented in order to identify the empty concept. Consequently, $\TD(\cC)=\sup\{\TD(C,\cC)\mid C\in\cC\}=m$.

As mentioned in the introduction, various notions of teaching have been proposed in the literature. The one that is most relevant to our work is the model of preference-based teaching. In this model, intuitively, a preference relation on $\cC$ is used to reduce the size of teaching sets. In particular, a concept $C$ need no longer be the only concept consistent with its 
teaching set 
$T(C)$; it suffices if $C$ is the unique most preferred concept in $\cC$ that is consistent with $\cC$. In order to avoid cyclic preferences, the preference relation is required to form a partial order over $\cC$.

\begin{definition}[\cite{GaoRSZ17}]\label{def:pbtd}
Let $\cC$ be a concept class over a domain $\cX$ and $\succ$ any binary relation that forms a strict (possibly non-total) order over $\cC$. We say that concept $C$ is preferred over concept $C'$ (with respect to $\succ$), if $C\succ C'$.  The {\em preference-based teaching dimension} of $C$ with respect to $\cC$ and $\succ$, denoted by $\PBTD(C,\cC,\succ)$, is the size of the smallest sample set $S$ such that 
\begin{enumerate}
    \item $S$ is consistent with $C$, and
    \item $C\succ C'$ for all $C'\in\cC\setminus\{C\}$ such that $S$ is consistent with $C'$.
\end{enumerate}
We write
$\PBTD(\cC,\succ)=\sup\{\PBTD(C,\cC,\succ)\mid C\in\cC\}$. Finally, the preference-based teaching dimension of $\cC$, denoted by $\PBTD(\cC)$, is 
defined by 
\[\PBTD(\cC)=\min\{\PBTD(\cC,\succ)\mid\,\succ\subseteq\cC\times\cC\mbox{ and }\succ\mbox{ forms a strict order on }\cC\}\,.\]
\end{definition}

An interesting variant of preference-based teaching is obtained when disallowing negative examples in teaching. Learning from positive examples only has been studied extensively in the computational learning theory literature, see, e.g., \citep{Denis01,Angluin80} and is motivated by studies on language acquisition~\citep{WexlerC80} or, more recently, by problems of learning user preferences from a user's interactions with, say, an e-commerce system~\citep{SchwabPK00}, as well as by problems in bioinformatics~\citep{WangDMH06}.

\begin{definition}[\cite{GaoRSZ17}]\label{def:pbtd+}
Let $\cC$ be a concept class over a domain $\cX$. The {\em positive preference-based teaching dimension} of $\cC$, denoted by $\PBTD^+(\cC)$, is defined analogously to $\PBTD(\cC)$, where the sets $S$ in Definition~\ref{def:pbtd} are required to contain only positive examples.
\end{definition}

In the same way, one can define the notion $\TD^+$. The following property, proven by \cite{GaoRSZ17}, is crucial when computing the $\PBTD$ and $\PBTD^+$ of finite classes.

\begin{proposition}[\cite{GaoRSZ17}]\label{prop:pbtdfinite}
Let $\cC$ be a finite concept class. If $\PBTD(\cC)\!=\!d$, then $\cC$ contains some $C$ with $\TD(C,\cC)\le d$. If $\PBTD^+(\cC)\!=\!d$, then $\cC$ contains some $C$ with $\TD^+(C,\cC)\le d$.
\end{proposition}

This result immediately implies that $\PBTD$ and the well-known notion of $\RTD$\footnote{The RTD, short for ``recursive teaching dimension’’, is a well-studied teaching parameter defined by Zilles et al. (2011).} 
coincide for finite concepts classes, and so do $\PBTD^+$ and $\RTD^+$.

\section{Collusion-free Teaching and the Non-Clashing Property}

While there is no objective measure of how ``reasonable'' a formal model of teaching is, the literature offers some notions of what constitutes an ``acceptable'' model of teaching, i.e., one in which the teacher and learner do not collude. So far, the notion of collusion-free teaching that found the most positive resonance in the literature is the one defined by Goldman and Mathias.

\begin{definition}[\cite{GoldmanM96}]\label{def:collusion}
Let $\cC$ be a concept class over $\cX$ and $T$ a teacher mapping on $\cC$. Let $L$ be a learner mapping that assigns to each set of labelled examples a concept over $\cX$. The pair $(T,L)$ is successful on $\cC$ if $L(T(C))=C$ for all $C\in\cC$. The pair $(T,L)$ is {\em collusion-free} on $\cC$ if $L(S)=L(T(C))$ for any $C\in\cC$ and any set $S$ of labelled examples such that $S$ is consistent with $C$ and $S$ contains $T(C)$.
\end{definition}

\noindent Intuitively, Goldman and Mathias's definition captures the idea that a learner conjecturing concept $C$ will not change its mind when given additional information consistent with $C$.

For example, teacher-learner pairs following the classical teaching dimension model, Balbach's optimal teacher model, the recursive teaching model, or the preference-based teaching model are always collusion-free according to Definition~\ref{def:collusion}. Of these models, the classical teaching dimension model is the one imposing the most constraints on the mapping $T$, followed by Balbach's optimal teaching, recursive teaching, and preference-based teaching in that order. Consequently, the ``teaching complexity'' among these models is lowest for preference-based teaching; if every concept in a concept class $\cC$ can be taught with at most $z$ examples in any of these models, then every concept in $\cC$ can be taught with at most $z$ examples in the preference-based model.

One can still argue that the preference-based model is unnecessarily constraining. Preference-based teaching of a concept class $\cC$ relies on a preference relation that induces a strict order on $\cC$. However, this strict order is used by the learner only after the teaching set has been communicated, since the learner chooses the unique most preferred concept among those \emph{consistent with the set of examples provided by the teacher}. One might consider loosening the constraints by, for example, demanding only that the set of concepts consistent with any chosen teaching set be ordered under the chosen preference relation (rather than requiring acyclic preferences over the whole concept class). In the same vein, one could relax more conditions---every relaxation might result in a more powerful model of teaching satisfying the collusion-free property. 

In this manuscript, we will define the provably most powerful model of teaching that is collusion-free in the sense proposed by \cite{GoldmanM96}, namely a model that adheres to no other constraints on the teacher-learner pairs $(T,L)$ than those given by Goldman and Mathias: (i) $T$ is a teacher mapping; (ii) $(T,L)$ is successful on $\cC$; and (iii) $(T,L)$ is collusion-free on $\cC$.

Before we define this model formally, we introduce 
a crucial property
that was originally proposed by \cite{KW2007} in the
context of unlabeled sample compression.

\begin{definition}
Let $\cC$ be a concept class and $T$ be a teacher mapping on $\cC$. 
We say that $T$ is {\em non-clashing} (on $\cC$) if and only if there are no two distinct $C,C'\in\cC$ such that both $T(C)$ is consistent with $C'$ and $T(C')$ is consistent with $C$.
\end{definition}

It turns out that, for a teacher mapping $T$, the non-clashing property is equivalent to the existence of a learner mapping $L$ such that $(T,L)$ is successful and collusion-free:

\begin{theorem}\label{thm:nc}
Let $\cC$ be a concept class over the instance space $\cX$. Let $T$ be a teacher mapping on $\cC$. Then the following two conditions are equivalent:
\begin{enumerate}
\item $T$ is non-clashing on $\cC$.
\item There is a mapping $L:2^{\cX\times\{0,1\}}\rightarrow\cC$ such that $(T,L)$ is both successful and collusion-free on $\cC$. 
\end{enumerate}
\end{theorem}

\begin{proof}
First, suppose $T$ is a non-clashing teacher mapping, and define $L$ as follows. Given any set $S$ of labelled examples as input, $L$ checks for the existence of a concept $C\in\cC$ such that $T(C)\subseteq S$ and $C$ is consistent with $S$. If such a concept $C$ is found, $L$ returns an arbitrary such $C$; otherwise $L$ returns some default concept in $\cC$. 

To show that $(T,L)$ is successful and collusion-free, suppose there is some concept $C\in\cC$ such that a given set $S$ of labelled examples is consistent with $C$ and contains $T(C)$. We claim that then such $C$ is uniquely determined. For if there were two distinct concepts $C,C'\in\cC$ consistent with $S$ such that $T(C)\cup T(C')\subseteq S$, then $T(C')$, being a subset of $S$, would be consistent with $C$ and, likewise, $T(C)$ would be consistent with $C'$---in contradiction to the non-clashing property of $T$. From the definition of $L$, it then follows that $(T,L)$ is successful and collusion-free.

Second, suppose $T$ is a teacher mapping and there is a mapping $L$ such that $(T,L)$ is successful and collusion-free, i.e., for all $C\in\cC$, we have $L(S)=L(T(C))=C$ whenever $S$ is consistent with $C$ and contains $T(C)$. To see that $T$ is non-clashing, suppose two concepts $C,C'\in\cC$ are both consistent with $T(C)\cup T(C')$. Then
$C=L(T(C))=L(T(C)\cup T(C'))=L(T(C'))=C'$. 
\end{proof}

Consequently, teaching with non-clashing teacher mappings is, in terms of the worst-case number of examples required, the most efficient model that obeys Goldman and Mathias's notion of collusion-freeness. We hence define the notion of no-clash teaching dimension as follows.

\begin{definition}
Let $\cC$ be a concept class over the instance space $\cX$. Let $T:\cC\rightarrow(\cX\times\{0,1\})^*$ be a non-clashing teacher mapping. 
The {\em order} of $T$ on $\cC$, denoted by $\TC(T,\cC)$, is then defined by 
$\TC(T,\cC)=\sup\{|T(C)|\mid C\in\cC\}$.
The {\em No-Clash Teaching Dimension} of $\cC$, denoted by $\NCTD(\cC)$, is defined as $\NCTD(\cC)=\min\{\TC(T,\cC)\mid T\mbox{ is a non-clashing teacher mapping for }\cC\}$.
\end{definition}

From Theorem~\ref{thm:nc} we obtain that, for every concept class $\cC$,
\[
\NCTD(\cC)=\min\{\TC(T,\cC)\mid \mbox{ there exists an }L\mbox{ s.t. }(T,L)\mbox{ is successful and collusion-free on }\cC\}.
\]

As in the case of preference-based teaching, it is natural to study a variant of non-clashing teaching that uses positive examples only.

\begin{definition}\label{def:nctd+}
Let $\cC$ be a concept class over the domain $\cX$. A teacher mapping $T$ is called positive on $\cC$ if $T(C)\subseteq \cX\times\{1\}$ for all $C\in\cC$. We then define 
$\NCTD^+(\cC)=\min\{\TC(T,\cC)\mid T\mbox{ is a positive non-clashing teacher mapping for }\cC\}$.
\end{definition}

Furthermore, for finite domains $\cX$, it will be helpful to have the notion of average no-clash teaching dimension:

\begin{definition}\label{def:anctd}
Let $\cC$ be a concept class over the finite domain $\cX$. The {\em Average No-Clash Teaching Dimension} of $\cC$, denoted by $\ANCTD(\cC)$, is defined as
\[
\ANCTD(\cC) = \min\left\{\left.\frac{1}{|\cC|}\sum_{C\in\cC}|T(C)|
\right| \mbox{ $T$ is a non-clashing teacher mapping for $\cC$} \right\}
\enspace .
\]
\end{definition}

{\remark \label{rem:(A)NCTD}
{\rm
It follows immediately from the pigeon-hole principle that 
$\NCTD(\cC) \ge \lceil \ANCTD(\cC) \rceil$.
}
}

In the following we describe a natural \emph{normal form} for non-clashing teacher mappings.
$T'$ is said to be an {\em extension} of $T$ if $T(C) \seq T'(C)$
holds for every $C\in\cC$. Clearly, if $T'$ is an extension of $T$
and $T$ is non-clashing, then $T'$ is non-clashing.

\begin{proposition}\label{prop:normalform}
(a) Let $T$ be a non-clashing teacher mapping for $\cC$. Then there is a non-clashing teacher mapping $T'$ for $\cC$ such that $|T'(C)| = \TC(T,\cC)$ for all $C\in\cC$.\\
(b) Let $T$ be a positive non-clashing teacher mapping for $\cC$. Then there is a positive non-clashing teacher mapping $T'$ for $\cC$ such that $|T'(C)| = \min \{ |C|, \TC(T,\cC) \}$ for all $C\in\cC$.
\end{proposition}

While many of our definitions and results apply to both finite and infinite concept classes, 
except where explicitly stated otherwise,
we will
hereafter
assume that $\cX$ (and $\cC$) are finite.

\section{Lower Bounds on $\NCTD$ and $\NCTD^+$}\label{sec:lb}

To establish lower bounds on $\NCTD$ and $\NCTD^+$ for finite concept classes, we first show that $\NCTD(\cC)$ must be at least as large as the smallest $d$ satisfying $|\cC| \le 2^d {|\cX| \choose d}$. A similar statement then follows for $\NCTD^+$. In fact, we prove a slightly stronger result, replacing $|\cX|$ with a potentially smaller value: 

\begin{definition}
We define $\cX_T\seq\cX$
as the set of instances that are part of a labelled example
in a teaching set $T(C)$ for some $C\in\cC$.
Moreover, we define 
\[ 
X(\cC) = 
\min\{|\cX_T|:\ \mbox{$T$ is a non-clashing teacher mapping for $\cC$ with }\TC(\cC,T) = \NCTD(\cC)\}
\enspace .
\]
\end{definition}

Intuitively, $X(\cC)$ is the smallest number of instances that must be employed by any optimal non-clashing teacher mapping for $\cC$. Likewise, we define $X^+(\cC)$ for positive non-clashing teaching.

\begin{theorem} \label{th:ub-size}
Let $\cC$ be any concept class.
\begin{enumerate}
    \item If $\NCTD(\cC)=d$, then $|\cC| \le 2^d {X(\cC) \choose d}$.
    \item If $\NCTD^+(\cC)=d$, then $|\cC| \le \sum_{i=0}^d {X^+(\cC) \choose i}$.
\end{enumerate}
\end{theorem}

\begin{proof}
To prove statement 1, let $\cX'$ be a subset of size $X(\cC)$ of $\cX$. 
Let $C \mapsto T(C) \seq \cX'\times\{0,1\}$ be a consistent 
and non-clashing mapping which witnesses that $\NCTD(\cC)=d$, 
and let $L$ be the mapping such that $L(T(C)) = C$ 
for all $C\in\cC$. By Proposition~\ref{prop:normalform}, one may assume 
without loss of generality that $|T(C)|=d$ for all $C\in\cC$. 
Since $T$ is an injective mapping and there are 
only $2^d {X(\cC) \choose d}$ labelled teaching sets at our 
disposal, the claim follows.

Statement 2 is proven analogously, taking into consideration that, in the $\NCTD^+$ case, we do not have an analogous statement to Proposition~\ref{prop:normalform}, since a concept does not in general contain $d$ or more elements. Note that the formula has no factors $2^i$ since there are no options for labelling the instances in any set $T(C)$.
\end{proof}


We will next establish a useful lower bound on $\NCTD(\cC)$,
as well as as a related lower bound on $\NCTD^+(\cC)$, based on the number of neighbors of any concept in $\cC$. 

A concept $C'\in\cC$ is a {\em neighbor} of concept $C\in\cC$ if it differs from $C$ on exactly one instance, i.e., if
the symmetric difference
$C \Delta C' := (C \setminus C') \cup (C' \setminus C)$ has size one.
The {\em degree} of $C\in\cC$, denoted as $\deg_\cC(C)$, 
is defined as the number of neighbors of $C$ in $\cC$. 
The average degree of concepts in $\cC$ is then denoted by
\[ 
\deg_{avg}(\cC) := \frac{1}{|\cC|}\cdot \sum_{C\in\cC}\deg_\cC(C)\,.
\]
The {\em dominance} of $C\in\cC$, denoted as $\dom_\cC(C)$, 
is defined as the number of smaller neighbors of $C$ in $\cC$, i.e. neighbors that contain exactly one fewer instance than $C$. 

\begin{theorem} \label{th:degree-lb}
Every concept class $\cC$ over a finite domain 
satisfies (i)
$\ANCTD(\cC) \ge \frac{1}{2} \cdot \deg_{avg}(\cC)$\\
and (ii) $\NCTD(\cC) \ge \lceil \frac{1}{2} \cdot \deg_{avg}(\cC) \rceil$.
\end{theorem}

\begin{proof}
For assertion (i), let $T$ be any non-clashing teacher mapping for $\cC$. If $C_1$ and $C_2$ 
are neighbors, say $C_1 \Delta C_2 = \{x_i\}$, then at least one 
of the sets $T(C_1),T(C_2)$ must contain $x_i$. We obtain
$ 
\sum_{C\in\cC}|T(C)| \ge \frac{1}{2}\cdot\sum_{C\in\cC}\deg_\cC(C)
= |\cC|\cdot\frac{1}{2} \cdot \deg_{avg}(\cC)
$. Assertion (ii) is immediate from (i), by Remark~\ref{rem:(A)NCTD}.
\end{proof}


\begin{theorem} \label{th:degree+-lb}
Every concept class $\cC$ over a finite domain 
satisfies 
$\NCTD^+(\cC) \ge \max_{C\in\cC}\dom_\cC(C)$.
\end{theorem}

\begin{proof}
If the smaller neighbor $C'$ of $C \in \cC$ differs from $C$ on instance $x_i$, then 
$(x_i,1)$ must be used in teaching $C$.
Hence, every $C \in \cC$ must have a positive teaching set of size at least $\dom_\cC(C)$.
\end{proof}


Although the lower bounds in Theorems~\ref{th:degree-lb} and \ref{th:degree+-lb} are not 
expected to be attained very often, the following Remark shows that they are sometimes tight:

{\remark \label{rem:powerset}
{\rm
Let $\cP_m$ be the  
powerset over the domain $\{x_1,\ldots,x_m\}$.
Since every concept in $\cP_m$ has degree $m$, clearly $\deg_{avg}(\cP_2) =m$. It follows from Theorem~\ref{th:degree-lb} that 
$\ANCTD(\cP_m) \ge  m/2$ and hence $\NCTD(\cP_m) \ge  \lceil m/2 \rceil $.
Furthermore, since $\dom_{\cP_m}(\{x_1,\ldots,x_m\})= m$, it follows from Theorem~\ref{th:degree+-lb} that $\NCTD^+(\cP_m) \ge m$. But the positive mapping $T$ that maps $S \in \cP_m$ to $S \times \{1\}$ is trivially non-clashing, and hence $\NCTD^+(\cP_m) = m$ and $\ANCTD(\cP_m) =  m/2$.
As the mapping $T$ given by
\[
\eset\mapsto\{(a,0)\} , \{a\}\mapsto\{(b,0)\} ,
\{b\}\mapsto\{(b,1)\} , \{a,b\}\mapsto\{(a,1)\} \enspace ,
\]
is non-clashing for $\cP_2$, it follows that $\NCTD(\cP_2) = 1$.
As we shall see in Theorem~\ref{th:powerset}, this generalizes to 
$\NCTD(\cP_m)= \lceil m/2 \rceil$.
}
}



We note that the maximum degree of a concept in $\cC$ is in general not an upper bound on $\NCTD(\cC)$. For example, if we consider the concept class $\cC$ consisting of subsets of size $k$ of some domain of size $n$, then all concepts in $\cC$ have degree zero yet, for $n$ sufficiently large, $\NCTD(\cC) = k$
(since for large enough $n$ the size of $\cC$ exceeds the number of possible teaching sets in a normal-form teaching mapping $T$ for $\cC$ with $\TC(T,\cC) < k$.)


\section{Sub-additivity of $\NCTD$ and $\NCTD^+$}\label{sec:ub}

In this section, we will show that the $\NCTD$ is sub-additive with respect to the free combination of concept classes. As an application of this result, we will determine the $\NCTD$ of the powerset over any finite domain $\cX$. While the powerset is a rather special concept class, knowing its $\NCTD$ will turn out useful to obtain a variety of further results.

\begin{definition}
Let $\cC_1$ and $\cC_2$ be concept classes over disjoint 
domains $\cX_1$ and $\cX_2$, respectively. Then the {\em free combination} $\cC_1\sqcup\cC_2$ of $\cC_1$ and $\cC_2$ is a concept class over the domain $\cX_1\cup\cX_2$
defined by $\cC_1\sqcup\cC_2 = 
\{C_1 \cup C_2|\ C_1\in\cC_1\mbox{ and } C_2\in\cC_2\}$.
\end{definition}


\begin{lemma} \label{lem:nctd-subadditive}
Let $\cC = \cC_1\sqcup\cC_2$ be the free combination of $\cC_1$
and $\cC_2$. Moreover, for $i=1,2$, let $T_i$ be a non-clashing mapping 
for $\cC_i$. Then, for $T(\cC_1\sqcup\cC_2)$ defined by setting
$ T(C_1 \cup C_2) = T_1(C_1) \cup T_2(C_2)$,
we have that
$T$ is a non-clashing teacher mapping 
for $\cC_1\sqcup\cC_2$. Moreover, as witnessed
by $T$, $\NCTD$ acts sub-additively on $\sqcup$, i.e.,
\begin{equation} \label{eq:nctd-subadditive}
\NCTD(\cC_1\sqcup\cC_2) \le \NCTD(\cC_1)+\NCTD(\cC_2) \enspace .
\end{equation}
\end{lemma}

\begin{proof}
Suppose that concepts $C_{i_1}, C_{j_1} \in \cC_1$ and $C_{i_2}, C_{j_2} \in \cC_2$ give rise to distinct concepts $C_{i_1} \cup C_{i_2}$ and  $C_{j_1} \cup C_{j_2} \in \/ \cC_1\sqcup\cC_2$ that clash under $T$.
%
(Without loss of generality we can assume that
$i_1 \neq j_1$.) Then $C_{j_1} \cup C_{j_2}$
is consistent with $T_1(C_{i_1}) \cup T_2(C_{i_2})$ and
$C_{i_1} \cup C_{i_2}$
is consistent with $T_1(C_{j_1}) \cup T_2(C_{j_2})$. Hence 
$C_{j_1}$ is consistent with $T_1(C_{i_1})$ and 
$C_{i_1}$ is consistent with $T_1(C_{j_1})$, that is concepts 
$C_{i_1}$ and $C_{j_1}$ in $\cC_1$ clash under the mapping $T_1$.
\end{proof}

\noindent
As we shall see NCTD sometimes acts strictly sub-additively on $\sqcup$; in particular, the composition of optimal mappings for $\cC_1$ and $\cC_2$ is not necessarily an optimal mapping for $\cC_1 \sqcup \cC_2$.
In contrast, ANCTD acts 
additively on $\sqcup$:

\begin{lemma} \label{lem:anctd-additive}
Let $\cC = \cC_1\sqcup\cC_2$ be the free combination of $\cC_1$
and $\cC_2$. Moreover, for $i=1,2$, let $T_i$ be a non-clashing mapping 
for $\cC_i$. Then, for $T(\cC_1\sqcup\cC_2)$ defined by setting
$ T(C_1 \cup C_2) = T_1(C_1) \cup T_2(C_2)$,
we have that
$T$ is a non-clashing teacher mapping 
for $\cC_1\sqcup\cC_2$. Moreover, as witnessed
by $T$, $\ANCTD$ acts additively on $\sqcup$, 
i.e., 
\begin{equation} \label{eq:anctd-subadditive}
\ANCTD(\cC_1\sqcup\cC_2) = \ANCTD(\cC_1)+\ANCTD(\cC_2) \enspace .
\end{equation}
\end{lemma}

\begin{proof}
The proof of Lemma~\ref{lem:nctd-subadditive} above shows that ANCTD acts 
sub-additively on $\sqcup$, 
that is\\ 
$\ANCTD(\cC_1\sqcup\cC_2) \le \ANCTD(\cC_1)+\ANCTD(\cC_2)$. 
It remains to show that 
\begin{equation} \label{eq:anctd-sup-additive}
\ANCTD(\cC_1\sqcup\cC_2) \ge \ANCTD(\cC_1)+\ANCTD(\cC_2) \enspace .
\end{equation}

To this end, let $\cX_1$ (resp. $\cX_2$) be the domain of concept class
${\cal C}_1$ (resp. ${\cal C}_2$)
and suppose that $T$ is a non-clashing teacher mapping 
	on $\cC$ such that $\ANCTD(\cC) = \frac{1}{|\cC|}\sum_{C\in\cC}|T(C)|$.
	The following calculation makes use of the fact, for every
	fixed choice of $C_2\in\cC_2$, the mapping 
	$C_1 \mapsto T(C_1 \cup C_2)\cap\cX_1$ is a non-clashing teacher
	mapping on $\cC_1$ (and an analogous remark holds, for reasons of
	symmetry, when the roles of $\cC_1$ and $\cC_2$ are exchanged):
	\begin{eqnarray*}
		\ANCTD(\cC) & = & 
			\frac{1}{|\cC_1|\cdot|\cC_2|}
		      \sum_{C_1\in\cC_1}\sum_{C_2\in\cC_2}|T(C_1 \cup C_2|) \\
		      & = & \frac{1}{|\cC_2|}\sum_{C\in\cC_2} 
\left(\frac{1}{|\cC_1|} \sum_{C\in\cC_1}|T(C_1 \cup C_2)\cap\cX_1|\right) \\
		&& \mbox{} + \frac{1}{|\cC_1|}\sum_{C\in\cC_1}
\left(\frac{1}{|\cC_2|} \sum_{C\in\cC_2}|T(C_1 \cup C_2)\cap\cX_2|\right) \\
        & \ge  &  \ANCTD(\cC_1)+\ANCTD(\cC_2)
			\enspace . 
	\end{eqnarray*}

\end{proof}


{\remark
{\rm In Lemma~\ref{lem:nctd-subadditive}, if $T_1$ and $T_2$ are positive non-clashing mappings, then the same proof shows that $T$ (a positive non-clashing mapping) witnesses the fact that $\NCTD^+$ also acts sub-additively on $\sqcup$, i.e., 
\begin{equation} \label{eq:nctd+-subadditive}
\NCTD^+(\cC_1\sqcup\cC_2) \le \NCTD^+(\cC_1)+\NCTD^+(\cC_2) \enspace .
\end{equation}
Furthermore, since $\sqcup$ is associative, it follows immediately that,
for any concept class $\cC$, 
if $\cC^k := \cC_1 \sqcup\ldots\sqcup \cC_k$, where 
$\cC_i := \{C \times \{i\}|\ C\in\cC\}$ for $i=1,\ldots,k$, then
\begin{equation}
\ANCTD(\cC^k) = k \cdot \ANCTD(\cC)
\end{equation}
and 
\begin{equation} \label{eq:concepts-power}
\NCTD(\cC^k) \le k \cdot \NCTD(\cC) \; {\rm and} \;
\NCTD^+(\cC^k) \le k \cdot \NCTD^+(\cC)
\enspace .
\end{equation}
}
}

We have already seen, in Remark~\ref{rem:powerset}, that  $\ANCTD(\cP_m) =  m/2 $, $\NCTD^+(\cP_m) = m$ and $\NCTD(\cP_m) \ge  \lceil m/2 \rceil $,
where $\cP_m$ denotes the powerset over the domain $\{x_1,\ldots,x_m\}$.
The sub-additivity results above can be applied in order to determine $\NCTD(\cP_m)$ exactly as well. 



\begin{theorem} \label{th:powerset}
Let $\cP_m$ be the powerset over the domain $\{x_1,\ldots,x_m\}$.
Then 
$\NCTD(\cP_m) =  \lceil m/2 \rceil $.
\end{theorem}

\begin{proof}
It remains to show that $\NCTD(\cP_m) \le \lceil m/2 \rceil$.
It suffices to verify this upper bound for even $m$. 
But, when $m$ is even,
$\NCTD(\cP_m) = \NCTD(\cP_2^{m/2}) \le m/2$ follows
from~(\ref{eq:concepts-power}) and the fact that 
$\NCTD(\cP_2)=1$ (cf. Remark~\ref{rem:powerset}).
 
\end{proof}



Since the $\NCTD$ of any concept class over a domain $\cX$ is trivially upper bounded by the 
$\NCTD$ of the powerset over $\cX$,
this result in particular implies that $\lceil|\cX|/2\rceil$ is an upper bound on the $\NCTD$ of any concept class over a domain $\cX$.

A further consequence of Theorem~\ref{th:powerset} is that $\NCTD$ is 
sometimes strictly subadditive
with respect to free combination, i.e., that inequality~(\ref{eq:nctd-subadditive}) is sometimes strict. 
 An example for that is the free combination $\cP_m\sqcup\cP_m$ of two copies of $\cP_m$ for odd $m$. Since the domain of $\cP_m\sqcup\cP_m$ has size $2m$, we obtain $\NCTD(\cP_m\sqcup\cP_m)=m$, while $\NCTD(\cP_m)+\NCTD(\cP_m)=2\lceil\frac{m}{2}\rceil = 2\frac{m+1}{2}=m+1$.



Another situation (that we will exploit later) where 
$\NCTD^+$ acts strictly additively on $\sqcup$, is captured in the following:
\begin{lemma}\label{lem:powercombination+}
Let $\cP_m$ be the powerset over the domain $\{x_1,\ldots,x_m\}$ and
let $\cC$ be a concept class with domain $\cX$ disjoint from $\{x_1, \ldots, x_{m} \}$.
Then, 
\[
\NCTD^+(\cP_{m} \sqcup \cC) = \NCTD^+(\cP_{m}) + \NCTD^+(\cC).
\]
\end{lemma}

\begin{proof}
By (\ref{eq:nctd+-subadditive}) it suffices to show that
$\NCTD^+(\cP_{m} \sqcup \cC) \ge \NCTD^+(\cP_{m}) + \NCTD^+(\cC)$.
Theorem~\ref{th:degree+-lb} implies that, for each  
$C_i \in \cC$, any positive non-clashing mapping $T$ for $\cP_{m} \sqcup \cC$ must use  $m = \NCTD^+(\cP_{m})$ 
examples from $\{x_1, \ldots, x_{m} \}$ to teach the single concept 
$\{x_1, \ldots, x_{m} \} \sqcup C_i$ within the concept class $\cP_{m} \sqcup C_i$. So the only way that $T$ could use fewer than  $m+ \NCTD^+(\cC)$ examples in total for each concept in $\{x_1, \ldots, x_{m} \} \sqcup \cC$
is if each such concept is taught with exactly $m$ examples from $\{x_1, \ldots, x_{m} \}$, and hence fewer than $\NCTD^+(\cC)$ examples from $\cX$, a contradiction. 
\end{proof}


Furthermore, it is easily seen that the average degree acts additively on $\sqcup$:

\begin{lemma}
Let $\cC_1$ and $\cC_2$ be concept classes over disjoint
and finite domains. Then the following holds: 
\begin{equation} \label{eq:avg-degree-add}
\deg_{avg}(\cC_1\sqcup\cC_2) = \deg_{avg}(\cC_1)+\deg_{avg}(\cC_2)
\enspace .
\end{equation}
\end{lemma}

\begin{proof}
Let $\cC := \cC_1\sqcup\cC_2$. The concepts in $\cC$ that
are neighbors of $C_1 \cup C_2 \in \cC$ are precisely the
concepts of the form $C_1 \cup C'_2$ or $C'_1 \cup C_2$
where $C'_2$ is a neighbor of $C_2$ in $\cC_2$ and $C'_1$
is a neighbor of $C_1$ in $\cC_1$. Hence
\[ 
\deg_\cC(C_1 \cup C_2) = 
\deg_{\cC_1}(C_1) + \deg_{\cC_2}(C_2)
\enspace .
\]
Moreover $|\cC| = |\cC_1|\cdot|\cC_2|$. It follows that
\[
\sum_{C\in\cC}\deg_\cC(C) = 
\sum_{C_1\in\cC_1}\sum_{C_2\in\cC_2}\deg_\cC(C_1 \cup C_2) =
|\cC_2|\cdot\sum_{C_1\in\cC_1}\deg_{\cC_1}(C_1) +
|\cC_1|\cdot\sum_{C_2\in\cC_2}\deg_{\cC_2}(C_2) \enspace .
\]
Division by $|\cC_1|\cdot|\cC_2|$ immediately
yields~(\ref{eq:avg-degree-add}).  
\end{proof}

The free combination of classes with a tight degree lower 
bound is again a class with a tight degree lower bound:

\begin{corollary} \label{cor:tight-lb}
Let $\cC_1$ and $\cC_2$ be two concept classes over disjoint
and finite domains, and let $\cC=\cC_1\sqcup\cC_2$. 
Then $\NCTD(\cC_i)=\frac{1}{2} \cdot \deg_{avg}(\cC_i)$
for $i=1,2$ implies that $\NCTD(\cC) = \frac{1}{2} \cdot \deg_{avg}(\cC)$.
\end{corollary}

\begin{proof}
The assertion is evident from the chain of 
inequalities:
\[ 
\NCTD(\cC) \stackrel{(\ref{eq:nctd-subadditive})}{\le} 
\NCTD(\cC_1)+\NCTD(\cC_2) = 
\frac{1}{2} \cdot \deg_{avg}(\cC_1) + 
\frac{1}{2} \cdot \deg_{avg}(\cC_2)  \stackrel{(\ref{eq:avg-degree-add})}{=} 
\frac{1}{2} \cdot \deg_{avg}(\cC) 
\enspace .
\]
and Theorem~\ref{th:degree-lb}.
\mbox{} 
\end{proof}

\section{Relation to Other Learning-theoretic Parameters}\label{sec:parameters}

In this section, we set $\NCTD$ in relation to $\PBTD$ and $\VCD$, as well as to the smallest possible size of a sample compression scheme for a given concept class.

\subsection{$\PBTD$ and $\VCD$}

Since preference-based teaching is collusion-free \citep{GaoRSZ17}, we obtain the following bounds.

\begin{proposition}\label{prop:ubPBTD}
Let $\cC$ be any concept class. Then $\NCTD(\cC)\le\PBTD(\cC)$ and $\NCTD^+(\cC)\le\PBTD^+(\cC)$.
\end{proposition}

{\remark{
\rm 
The first inequality in Proposition~\ref{prop:ubPBTD} is
sometimes strict, as witnessed by Theorem~\ref{th:powerset}, which states that  $\NCTD(\cP_m)=\lceil m/2\rceil$. By comparison, $\PBTD(\cP_m)=m$. In particular, this yields a family of concept classes of strictly increasing $\NCTD$ for which $\PBTD$ exceeds $\NCTD$ by a factor of 2. 
\noindent
The fact that the second inequality in Proposition~\ref{prop:ubPBTD} is sometimes strict is witnessed by the simple class  
$\cC$ described in the introduction, with $\NCTD^+(\cC)=1$. 
Since no concept in $\cC$ has a positive teaching set of size 1, Proposition~\ref{prop:pbtdfinite} implies $\PBTD^+(\cC)=2$.
\noindent
In particular, these examples witness that Proposition~\ref{prop:pbtdfinite} does \emph{not}\/ hold for non-clashing teaching.

}}

Results from the literature can now be combined in a straightforward way in order to formulate an upper bound on $\NCTD$ in terms of the VC-dimension. 

\begin{proposition}
$\NCTD(\cC)$ is upper-bounded by a function quadratic in $\VCD(\cC)$.
\end{proposition}

\begin{proof}
$\PBTD$ is known to lower-bound the recursive teaching dimension \citep{GaoRSZ17}. \cite{HuWLW17} proved that, when $\VCD(\cC)=d$, the recursive teaching dimension of $\cC$ is no larger than $39.3752\cdot d^2-3.6330\cdot d$. By Proposition~\ref{prop:ubPBTD}, the same upper bound applies to $\NCTD$.
\end{proof}

However, $\VCD$ can also be arbitrarily larger than $\NCTD$,
a result that follows immediately from the corresponding result for $\TD$:

\begin{proposition}[\cite{GK1995}]\label{prop:nctd<vcd}

Let $k\in\mathbb{N}$, $k\ge 1$. Then there exists a finite concept class $\cC$ such that
$\TD^+(\cC) = \TD(\cC) = 1$ and $\VCD(\cC) = k$. 
\end{proposition}

So far, there is no concept class for which $\VCD$ is known to exceed $\NCTD$. Note that any such concept class would have to fulfill $\PBTD>\VCD$ as well. We tested those classes for which $\PBTD>\VCD$ is known from the literature, but found that all of them satisfy $\NCTD\le\VCD$. 

As an example, here we present ``Warmuth's class.'' This concept class, shown in Table~\ref{tab:warmuthsclass}, was communicated by Manfred Warmuth and proven by \cite{DarnstadtKSZ16} to be the smallest concept class for which $\PBTD$ exceeds $\VCD$.  
In particular, $\VCD(\cC_W)\!=\!2$ while $\PBTD(\cC_W)\!=\!3$.

\begin{table}[ht]
\centering
\begin{tabular}{c|ccccc||c|ccccc}
&$x_1$&$x_2$&$x_3$&$x_4$&$x_5$&&$x_1$&$x_2$&$x_3$&$x_4$&$x_5$\\\hline
$C_1$&\textbf{1}&0&0&0&\textbf{1}&$C'_1$&\textbf{1}&0&\textbf{1}&0&1\\
$C_2$&\textbf{1}&\textbf{1}&0&0&0&$C'_2$&1&\textbf{1}&0&\textbf{1}&0\\
$C_3$&0&\textbf{1}&\textbf{1}&0&0&$C'_3$&0&1&\textbf{1}&0&\textbf{1}\\
$C_4$&0&0&\textbf{1}&\textbf{1}&0&$C'_4$&\textbf{1}&0&1&\textbf{1}&0\\
$C_5$&0&0&0&\textbf{1}&\textbf{1}&$C'_5$&0&\textbf{1}&0&1&\textbf{1}\\
\end{tabular}
\caption{Warmuth's class $\cC_W$, with the highlighted entries (in bold) corresponding to the images of a positive non-clashing teacher mapping. The domain of this class is $\{x_1,\ldots, x_5\}$, and it contains 10 concepts, named $C_1$ through $C_5$ and $C'_1$ through $C'_5$.}
\label{tab:warmuthsclass}
\end{table}

\begin{proposition}\label{prop:CW}
 $\NCTD(\cC_W)=\NCTD^+(\cC_W)=2$. 
\end{proposition}


\begin{proof}
The highlighted labels in Table~\ref{tab:warmuthsclass} correspond to a positive non-clashing mapping for $\cC_W$, which immediately shows that $\NCTD^+(\cC_W)\le 2$ and thus $\NCTD(\cC_W)\le 2$. To show that $\NCTD(\cC_W)\ge 2$, suppose by way of contradiction that $\NCTD(\cC_W)=1$. Then there is a non-clashing teacher mapping $T$ that assigns every concept in $\cC_W$ a teaching set of size 1.

Since $C_1$ and $C'_1$ differ only on the instance $x_3$, the mapping $T$ must fulfill either $T(C_1)=\{(x_3,0)\}$ or $T(C'_1)=\{(x_3,1)\}$.

Case 1. $T(C_1)=\{(x_3,0)\}$. Since $C_2$ is consistent with $T(C_1)$, the teaching set for $C_2$ must be inconsistent with $C_1$. In particular, $T(C_2)\ne\{(x_4,0)\}$. This implies $T(C'_2)=\{(x_4,1)\}$, since $x_4$ is the only instance on which $C_2$ and $C'_2$ disagree. By an analogous argument concerning $C_5$ and $C'_5$, one obtains $T(C'_5)=\{(x_2,1)\}$. Now $T$ has a clash on $C'_2$ and $C'_5$, which is a contradiction.

Case 2. $T(C'_1)=\{(x_3,1)\}$. One argues as in Case 1, with $C'_3$ and $C'_4$ in place of $C_2$ and $C_5$, yielding $T(C_3)=\{(x_5,0)\}$ and $T(C_4)=\{x_1,0)\}$. This is a clash, resulting in a contradiction.

As both cases result in a contradiction, we have $\NCTD(\cC_W)>1$ and thus $\NCTD(\cC_W)=2$. Since $\NCTD^+$ is an upper bound on $\NCTD$, we also have $\NCTD^+(\cC_W)=2$.
\end{proof}


While the general relationship between $\NCTD$ and $\VCD$ remains open, it turns out that $\NCTD(\cC)$ is upper-bounded by $\VCD(\cC)$ when $\cC$ is a finite \emph{maximum}\/ class. For a finite instance space $\cX$, a concept class $\cC$ of VC dimension $d$ is called maximum if its size $|\cC|$ meets Sauer's upper bound $\sum_{i=0}^d{|\mathcal{X}|\choose i}$ \citep{Sau1972} with equality. Recently, \cite{ChalopinCMW18} showed that every finite maximum class $\cC$ admits a so-called \emph{representation map}, i.e., a function $r$ that maps every concept in $\cC$ to a set of at most $d(=\VCD(\cC))$ instances, in a way that no two distinct concepts $C,C'\in\cC$ both agree on all the instances in $r(C)\cup r(C')$. By definition, any representation map is, translated into our setting, simply a non-clashing teacher mapping of order $d$ for $\cC$. Therefore, the result by Chalopin et al.\ implies that $\NCTD(\cC)\le \VCD(\cC)$ for finite maximum $\cC$.

\subsection{Sample Compression}

Intuitively, a sample compression scheme \citep{LW1986} for a (possibly infinite) concept class $\cC$ provides a lossless compression of every set $S$ of labeled examples for any concept in $\cC$ in the form of a subset of $S$. It was proven that the existence of a finite upper bound on the size of the compression sets is equivalent to PAC-learnability, i.e., to finite VC-dimension~\citep{MoranY16,LW1986}. Open for over 30 years now is the question how closely such an upper bound can be related to the VC-dimension.

Formally, a sample compression scheme of size $k$ for a concept class $\cC$ over $\cX$ is a pair $(f,g)$ of mappings, where, for every sample set $S$ consistent with some concept $C\in\cC$, (i) $f$ maps $S$ to a subset $f(S)\subseteq S$ with $|f(S)|\le k$; and (ii) $g(f(S))$ maps the compressed set to a concept $C'$ over $\cX$ (not necessarily in $\cC$) that is consistent with $S$.
By $\CN(\cC)$ we denote the size of the smallest-size sample compression scheme for $\cC$.
The open question then is whether $\CN(\cC)$ is upper-bounded by (a function linear in) $\VCD(\cC)$. 

Some connections between sample compression and teaching have been established in the literature~\citep{DFSZ2014,DarnstadtKSZ16}. The non-clashing property bears some similarities to sample compression and has in fact been used in the context of \emph{unlabelled}\/ sample compression (in which $f(S)$ is an unlabelled set) \citep{KW2007,ChalopinCMW18}. It is thus natural to ask whether $\CN$ is an immediate upper or lower bound on $\NCTD$. Below, we answer this question negatively.

\begin{proposition} \label{prop:sc}
\mbox{}
\begin{enumerate}
    \item For every $k\in\mathbb{N}$, $k\ge 1$, there is a concept class $\cC$ such that $\NCTD(\cC)=\PBTD(\cC)=1$ but $\CN(\cC)>k$.
    \item Let $\cP_m$ be the powerset over a domain of size $m$, where $m\ge 5$ is odd. Then $\CN(\cP_m) < \NCTD(\cP_m)$ and $2\CN(\cP_m) <\PBTD(\cP_m)$.
\end{enumerate}
\end{proposition}

\begin{proof}
Statement 1 is due to Remark~\ref{prop:nctd<vcd}, which implies the existence of a concept class $\cC$ with $\NCTD(\cC)=\PBTD(\cC)=1$ and $\VCD(\cC)=5k$. Then $\CN(\cC)>k$ follows from a result by \cite{FW1995} that states that no concept class of VC-dimension $d$ has a sample compression scheme of size at most $\frac{d}{5}$.

Statement 2 follows from the obvious fact that $\PBTD(\cP_m)=m$, in combination with Theorem~\ref{th:powerset}, as well as with a result by \cite{DarnstadtKSZ16} that shows $\CN(\cP_m)\le\lfloor\frac{m}{2}\rfloor$, for any $m \ge 4$.
\footnote{When $m=5k$ for some $k\ge 1$, \cite{DarnstadtKSZ16} even show that $\CN(\mathcal{P}_m)\le 2k$; hence there is a family of concept classes with $\CN<\NCTD$ for which the gap between $\CN$ and $\NCTD$ grows linearly with the size of the instance space.}
\end{proof}


Note that the compression function $f$ in a sample compression scheme for $\cC$ trivially induces a teacher mapping $T_f$ defined by $T_f(C)=f( \{ (x,C(x)) \mid x \in X \} )$. The decompression mapping $g$ then satisfies $g(T_f(C))=C$ for all $C \in \cC$. Hence $(T_f,g)$ is a successful teacher-learner pair. Proposition~\ref{prop:sc}.2 now states that there are concept classes for which the teacher-learner pairs $(T_f,g)$ induced by any optimal sample compression scheme necessarily display collusion. In other words, optimal sample compression yields collusive teaching.
An interesting problem is to find more examples of concept classes for which optimal sample compression yields collusive teachers and to determine necessary or sufficient conditions on 
the structure of such classes. Moreover, at present we do not know how large the gap between sample compression scheme size and NCTD can be.

As mentioned above, representation maps, which were proposed by \cite{KW2007} and \cite{ChalopinCMW18}, yield non-clashing teacher mappings. Clearly, in unlabelled compression, the representation map that compresses any concept in a class $\cC$ to a subset of $\cX$ must be injective, so that any two concepts in $\cC$ remain distinguishable after compression. In other words, the non-clashing teacher mappings induced by representation maps are \emph{repetition-free}, i.e., they do not map any two distinct concepts $C,C'\in\cC$ to labelled samples $T(C),T(C')$ for which 
\[
\{x\in\cX\mid (x,l)\in T(C)\mbox{ for some }l\in\{0,1\}\}\ne \{x\in\cX\mid (x,l')\in T(C')\mbox{ for some }l'\in\{0,1\}\}\enspace .
\]
Requiring no-clash teacher mappings to be repetition-free would be a limitation, as the example of the powerset over any set of $m$ instances, $m\ge 2$, shows. In this case, no-clash teaching can be done with teacher mappings of order $\lceil\frac{m}{2}\rceil$, but it is not hard to see that the best possible repetition-free no-clash teacher mapping is of order $m$.

\section{Complexity of Decision Problems Related to No-clash Teaching}


In this section, we address the complexity of the problem of deciding whether or not every concept in a given finite concept class can be taught with a non-clashing teaching set of size at most $k$, for some specified $k \ge 1$. Surprisingly perhaps, such decision problems are 
NP-hard, even when $k =1$ and teaching is done using positive examples only. 
In contrast, we show in subsection~\ref{sub:fastRTD} that the corresponding decision problems for $\PBTD$ (equivalently, for $\RTD$) 
have polynomial time solutions.

We show an equivalence between the most highly constrained such decision problem (testing if $\NCTD^+ =1$, for a given concept class) and a natural (but apparently not previously studied) constrained bipartite matching problem that is related to the well-studied notion of induced matchings.
%
The following, an immediate consequence of Proposition~\ref{prop:normalform},
allows us to restrict our complexity analysis to certain normalized concept classes.

\begin{proposition}
\label{prop:emptysetfree}
Let $\cC$ be any non-trivial concept class over a finite domain, with at least two non-empty concepts. Then,
$\NCTD^+(\cC) = \NCTD^+(\cC\setminus\{\emptyset\})$.
\end{proposition}
%
\begin{proof}
Let $\cC'$ denote $\cC\setminus\{\emptyset\}$. 
If $\cC' = \cC$ there is nothing to show. So, suppose that $\cC$ contains the empty concept. 
If $\NCTD^+(\cC) = k$ then trivially $\NCTD^+(\cC') \le k$. 

For the converse, suppose that $\NCTD^+(\cC')=k$, as witnessed by a normal-form mapping $T$ (cf. Proposition~\ref{prop:normalform}(b)).
Since $T$ does not assign the empty set to any concept one can obviously extend $T$ to assign the empty set to the empty concept and thus teach all of $\cC$ without clashes using no negative examples and with teaching sets of size at most $k$. (There are no clashes, because the empty concept cannot be consistent with any of the teaching sets that use at least one positive example.)
\end{proof}


Our goal in the remainder of this section is to set out hardness results for testing $\NCTD = k?$ and $\NCTD^+ = k?$, for fixed $k \ge 1$. We begin by establishing that testing $\NCTD^+ =1?$, for a given concept class $\cC$ is NP-hard. Other results follow by reduction from the $\NCTD^+ =1?$ decision problem. (It is straightforward to confirm that all of the decision problems $\NCTD \le k?$ and $\NCTD^+ \le k?$ are in NP.)

\subsection{Testing if ${\rm NCTD}^+ = 1$ is NP-hard}


Let $(\cC, \cX)$ be an instance of the ${\rm NCTD}^+ = 1$ decision problem. By 
Propositions~\ref{prop:normalform} and~\ref{prop:emptysetfree},
we can assume that $\cC$ does not contain the empty set, and that  positive teacher mappings realizing ${\rm NCTD}^+ = 1$   are restricted to those that use exactly one positive instance for each concept.

We start by observing that 
$(\cC, \cX)$
can be viewed as a bipartite graph $B_{{\cC}, {\cX}}$, with vertex classes $\cC$ (black vertices) and $\cX$ (white vertices) and an edge from $C_i \in {\cC}$ to $x_j \in {\cX}$ whenever $x_j \in C_i$. Under 
our assumptions, it follows 
that deciding if $\cal C$ has ${\NCTD}^+ = 1$  is equivalent to deciding if $B_{{\cal C}, {\cal X}}$ admits a matching $M$ such that (i) $M$ saturates all of the black vertices, and (ii) no two edges of $M$ are part of a $4$-cycle in $B_{{\cal C}, {\cal X}}$. (Condition (i) ensures that each concept in $\cal C$ has an associated positive teaching set of size $1$, and condition (ii) ensures that the resulting teacher mapping is non-clashing.)

We refer to the problem of deciding if a given bipartite graph $B$ with vertex partition $(V_b, V_w)$  admits a matching $M$ such that (i) $M$ saturates all of the vertices in $V_b$, and (ii) no two edges of $M$ are part of a $4$-cycle in $B$, as the {\em Non-Clashing Bipartite Matching Problem.}
The NP-hardness of deciding $\NCTD =1 ?$ is thus an immediate consequence of the following:

\begin{theorem}
\label{thm:NCmatching}
The Non-Clashing Bipartite Matching Problem is NP-hard.
\end{theorem}

The proof of Theorem~\ref{thm:NCmatching} is by reduction from the familiar NP-hard problem $3$-SAT. The details are given in Appendix~\ref{app:proof:NCmatching}.

{\remark
{\rm 
The reduction produces a bipartite graph whose vertices have degree bounded by five. 
One can conclude then that testing $\NCTD^+ =1$ is NP-hard even if concepts contain at most five instances, and instances are contained in at most five concepts. 
It is natural to ask to what extent this can be tightened. 
In Appendix~\ref{app:tighterbounds}.1, we describe a modification of the reduction that produces a bipartite graph whose vertices have degree bounded by three, from which it follows that testing $\NCTD^+ =1$ is NP-hard even if concepts contain at most three instances, and instances are contained in at most three concepts. 
On the other hand, if either (i) all concepts have at most two instances, or (ii) all instances are contained in at most two concepts,
the bipartite graph $B_{{\cC}, {\cX}}$ has the property that the degree of all vertices in one of its two parts bounded by at most two. In this case, it follows immediately from the algorithm in Appendix~\ref{app:tighterbounds}.2 that testing $\NCTD^+ =1$ can be done in polynomial time.

}}


\subsection{Testing if ${\NCTD} = 1$ is NP-hard}\label{sub:NCTD=1}

We reduce the ${\NCTD}^+ = 1$ decision problem to the ${\NCTD} = 1$ decision problem. 
Let $(\cal C, \cal X)$ be an instance of the ${\rm NCTD}^+ = 1$ decision problem. 
As before, we will assume (following 
Proposition~\ref{prop:emptysetfree})
that $\cC$ does not contain the empty set.
%
%
We make two disjoint copies $({\cal C}^1, {\cal X}^1)$ and
$({\cal C}^2, {\cal X}^2)$ of $(\cal C, \cal X)$, and take their union, denoted $2\cC$ to be an instance of the ${\rm NCTD} = 1$ decision problem. 
We argue that $\NCTD(2\cC) = 1$ if and only if $\NCTD^+(\cC) = 1$.

It is clear that $\NCTD^+(\cC) = 1$ implies $\NCTD(2\cC) = 1$.
For the converse,
suppose that teaching mapping $T$ provides a ${\rm NCTD} = 1$ solution of the concept class $2\cC$ that, among all such mappings, uses the fewest negative examples. 

Note that, all concepts in one component concept class are consistent with (necessarily negative) examples drawn from the opposite domain, and inconsistent with positive examples drawn from the opposite domain. Thus, the minimality of $T$ ensures that negative examples used for any component concept class $\cC^i$ are drawn from its associated domain $\cX^i$ (otherwise any such negative example could be replaced by a positive example for the corresponding concept, without creating a clash).

But, for the same reason, it cannot be that for both concept classes $\cC^i$ there exist one or more concepts whose teaching set uses a negative example drawn from the associated domain $\cX^i$, since any pair of concepts from different classes taught in this way would necessarily clash.
It follows that $T$ must use only positive examples (necessarily from the associated domain $\cX^i$) for teaching  concepts in at least one of the two component concept classes $\cC^i$; in this sense it must provide a ${\rm NCTD}^+ = 1$ solution of the instance $(\cal C, \cal X)$.


\subsection{Testing if ${\NCTD}^+ = k$ is NP-hard, for $k > 1$.}
Again we describe a reduction from the ${\NCTD}^+ = 1$ decision problem.
Given an instance of the ${\rm NCTD}^+ = 1$ decision problem, specifically a pair $(\cal C, \cal X)$, where $\cal C$ is a concept class  over the finite domain $\cal X$ disjoint from $\{x_1, \ldots, x_{k-1} \}$, we construct the concept class 
$\cP_{k-1} \sqcup \cC$.
By Lemma~\ref{lem:powercombination+}, we 
know that $\NCTD^+(\cP_{k-1} \sqcup \cC) = k-1 + \NCTD^+(\cC)$, so 
$\NCTD^+(\cC) = 1$ if and only if $\NCTD^+(\cP_{k-1} \sqcup \cC) = k$.

\subsection{Testing if ${\NCTD} = k$ is NP-hard, for $k > 1$.}
Again we describe a reduction from the ${\NCTD}^+ = 1$ decision problem. Let $\cC$ be a concept class over the finite domain $\cal X$, disjoint from $\{x_1, \ldots, x_{2(k-1)} \}$. We construct the composite concept class $4\cC := 2(2\cC)$ as in subsection~\ref{sub:NCTD=1}. By the reduction of that subsection, it will suffice to argue that $\NCTD(2\cC) = 1$ if and only if $\NCTD({\cal P}_{2(k-1)} \sqcup 4\cC) = k$.

First note that, by Theorem~\ref{th:powerset} and the sub-additivity of $\NCTD$ (equation (\ref{eq:nctd-subadditive})), $\NCTD({2\cC}) = 1$ implies  that 
$\NCTD({\cal P}_{2(k-1)} \sqcup 4{\cal C}) \le k$. 
In addition $\NCTD({2\cC}) = 1$ implies that $\ANCTD(4\cC) > 0$ which,
together with $\ANCTD({\cal P}_{2(k-1)}) = k-1$ (Remark~\ref{rem:powerset}) implies 
$\ANCTD({\cal P}_{2(k-1)}) + \ANCTD( 4\cC) > k-1$.
This in turn implies $\ANCTD({\cal P}_{2(k-1)} \sqcup 4\cC) >k-1$, by Lemma~\ref{lem:anctd-additive}, from which we immediately conclude that
$\NCTD({\cal P}_{2(k-1)} \sqcup 4{\cal C}) > k-1$, and hence
$\NCTD({\cal P}_{2(k-1)} \sqcup 4{\cal C}) = k$.

For the converse, we have seen that $\NCTD({\cal P}_{2(k-1)} \sqcup 4{\cal C}) = k$ implies $\ANCTD({\cal P}_{2(k-1)} \sqcup 4{\cal C}) \le k$ (trivially). This in turn implies
$\ANCTD({\cal P}_{2(k-1)})  +  \ANCTD(4{\cal C}) \le k$, by Lemma~\ref{lem:anctd-additive},
and hence
$\ANCTD(4{\cal C}) \le 1$, by Remark~\ref{rem:powerset}.
But $\ANCTD(4{\cal C}) \le 1$ implies $\NCTD(2{\cal C}) =1$, since
(i) $4\cC$ can be viewed as two copies of $2\cC$, and (ii) any teacher mapping for $4\cC$ realizing $\ANCTD(4{\cal C}) \le 1$ uses the empty set as a teaching set at most once, and hence uses a teaching set of size greater than 1 in at most one of the two copies of $2\cC$.

\subsection{Deciding $\PBTD \le k?$ (equivalently $\RTD \le k?$) has a polynomial-time solution}
\label{sub:fastRTD}

As before, we can cast the decision problem $\PBTD \le k?$ as a constrained matching problem in a suitably defined bipartite graph.

\begin{definition}
	Let $\cC$ be a concept class over domain $\cX$. 
	Let $1 \le k \le |\cX|$ be an integer. 
	Let $\cS_k = (\cX\times\{0,1\})^k$ be the family of labeled 
	samples of size $k$. Then $G_k=G_k(\cC)$ denotes the bipartite
	graph with vertex classes $\cC$ and $\cS_k$ and an edge between 
	$C\in\cC$ and $S\in\cS_k$ iff $C$ is consistent with $S$. 
\end{definition}

The following result is a slight extension of a well known result
in the theory of constrained matchings:

\begin{theorem} \label{th:ghl2001}
	Let $\cC$ be a concept class of size $m$ over domain $\cX$
	and let $G=G_k(\cC)$. Then the following statements are equivalent:
	\begin{enumerate}
		
		\item  
			There exists a matching $M$ of size $m$ in $G$ 
			such that $G$ contains no alternating cycles with 
			respect to $M$ (called a uniquely restricted matching of size $m$~\citep{GHL2001}).
		\item
			There exists a matching $M$ of size $m$ in $G$
			such that every $M$-induced subgraph contains a vertex
			of degree $1$. 
		\item
			There exist $m$ distinct vertices (samples) 
			$S_1,\ldots,S_m \in \cS_k$ and an ordering
			$C_1,\ldots,C_m$ of the vertices (=concepts) in $\cC$
			such that the following holds: 
			\begin{itemize}
				\item
					There is an edge between $C_i$ and $S_i$
					(i.e., $C_i$ is consistent with $S_i$).
				\item
					If there is an edge between $C_i$ and
					$S_j$ (i.e., if $C_i$ is consistent 
					with $S_j$), then $i \le j$.
			\end{itemize}
		\item
			$\PBTD(\cC) \le k$.
	\end{enumerate}
\end{theorem}

\begin{proof}
	1.$\Leftrightarrow$2. and 2.$\Leftrightarrow$3.  are well known 
	equivalences in the theory of constrained matchings. 
	See Theorem~$3.1$ in~\citep{GHL2001}. \\
	3.$\Rightarrow$4.:
	Under the specified interpretation, 
	the ordering over the concepts corresponds to a preference relation where $C_j$ is preferred over $C_i$ whenever $j>i$. This preference relation is a witness of $\PBTD(\cC)\le k$.\\
	4.$\Rightarrow$3.:
	If a preference relation (partial order) witnessing $\PBTD(\cC)\le k$ exists, then any linear extension of it will satisfy 3.
\end{proof}


{\remark \label{rem:c1991}
{\rm      The second statement in Theorem~\ref{th:ghl2001} can be strengthened 
	as follows: there exists a matching of size $m$ in $G$ such that 
	every $M$-induced subgraph contains a vertex of degree $1$ in each 
	of the two vertex classes, cf.~\citep{C1991}.
	
}}

The above characterization of classes with a recursive teaching dimension
of at most $k$ easily leads to a linear time algorithm for the corresponding 
decision problem:\footnote{See~\citep{C1991} for a similar algorithm
(based on a similar characterization) that decides in linear time whether 
a bipartite graph has a unique maximum matching.}

\begin{corollary}
	\begin{enumerate}
		\item
			Let $G$ be a bipartite graph with vertex classes
			$V_0$ and $V_1$ such that $|V_0| \le |V_1|$. 
			Let $|G|$ denote the size (= number of vertices 
			plus number of edges) of $G$. There is an algorithm 
			that runs in time $O(|G|)$ and returns a uniquely 
			restricted matching of size $|V_0|$ (provided it 
			exists).\footnote{The more general problem of
			deciding whether there exists a uniquely restricted 
			matching of size $k$, with $k$ being part of the input,
			is known to be NP-complete~\citep{GHL2001}.}
		\item 
			Suppose that the graph $G_k(\cC)$ associated with 
			a concept class $\cC$ is given. Then there is an 
			algorithm for checking whether 
			$\PBTD(\cC) \le k$
			whose run time is linear in the size of $G_k(\cC)$.
			Moreover, if $\PBTD(\cC) \le k$, it returns a
			preference relation that witnesses this fact.
	\end{enumerate}
\end{corollary}

The second part of the corollary is immediate from the first part
and Theorem~\ref{th:ghl2001}. The first part of the corollary is
based on the simple idea of initializing the matching $M$ with the empty
set and then iteratively doing the following:
\begin{enumerate}
	\item   
		If $V_0=\eset$, then return $M$ and stop.
	\item
		If $V_1$ does not contain any vertex of degree $1$,
		then return an error message (indicating that there
		exists no uniquely restricted matching of size $|V_0|$
		in $G$) and stop.\footnote{This can be justified by
		Remark~\ref{rem:c1991} and the invariance conditions 
		below.} Otherwise, pick a vertex of degree $1$
		from $V_1$, say vertex $v$ with $u$ as its unique neighbor
		in $V_0$.
	\item
		Insert the edge $(u,v)$ into $M$, remove $u$ from $V_0$
		and $v$ from $V_1$ and update the degrees of the vertices
		which are adjacent to either $u$ or $v$.
\end{enumerate}
Note that $V_0$ and $V_1$ are dynamically changed in the course 
of the algorithm. The following conditions are easily shown to be
satisfied after a run through the main loop provided that they are
satisfied immediately before entering the loop:
\begin{enumerate}
	\item
		There exists a uniquely restricted matching of size $|V_0|$
		for the subgraph induced by $V_0 \cup V_1$.
	\item
		Let $V'_0$ and $V'_1$ denote the vertices that have been
		removed from $V_0$ and $V_1$, respectively. Then there is 
		a unique perfect matching for the subgraph induced 
		by $V'_0 \cup V'_1$.
\end{enumerate}
The correctness of the algorithm directly follows from these
invariance conditions.

We briefly note that the results described in this section also hold,
mutatis mutandis, for $\PBTD^+$ in place of $\PBTD$.\footnote{The
corresponding bipartite graph has as its second vertex class the family
of positive samples of order at most $k$.}

\section{Conclusions}

No-clash teaching represents the limit of data efficiency that can be achieved in teaching settings obeying Goldman and Mathias's notion of collusion-freeness. Therefore, it is the sole most promising collusion-free teaching model to shed light on two open problems in computational learning theory, namely (i) to find a teaching complexity parameter that is upper-bounded by a function linear in $\VCD$, and (ii) to establish an upper bound on the size of smallest sample compression schemes that is linear in $\VCD$. 
If \emph{any}\/ collusion-free teaching model yields a complexity upper-bounded by (a function linear in) $\VCD$, then no-clash teaching does. Likewise, if \emph{any}\/ collusion-free model is powerful enough to compress concepts as efficiently as sample compression schemes do, then no-clash teaching is.

The most fundamental open question resulting from our paper is probably whether $\NCTD$ is upper-bounded by $\VCD$ in general.

Furthermore, our results introduce some intriguing connections between $\NCTD$ and the well-studied field of constrained matching in bipartite graphs 
that may open up a line of study that relates teaching complexity, as well as sample compression and $\VCD$, to fundamental issues in matching theory.

\bibliography{bib}

\appendix

\section{Proof of Theorem~\ref{thm:NCmatching}}\label{app:proof:NCmatching}

\begin{proof}
We describe a parsimonious reduction from the familiar NP-hard problem $3$-SAT, an instance of which is a set ${\cal D} =\{ D^1, \ldots, D^m \}$  of clauses, each of which is a disjunction of three literals drawn from an underlying set ${\cal V} = \{ V^1, \ldots, V^n \}$ of variables. Specifically, given an instance $\cal D$ of $3$-SAT, we construct a bipartite graph $B_{\cal D}$ (vertices are either black or white, and all edges join a black vertex to a white vertex) that admits a matching $M$ such that (i) $M$ saturates all of the black vertices, and (ii) no two edges of $M$ are part of a $4$-cycle in $B$, if and only if the instance $\cal D$ is satisfiable.

To this end, we first associate with each variable $V^i$ a \emph{variable gadget}: a ring of $4m$ vertices, with alternating subscripted labels $v^i$ and $w^i$, emphasizing its bipartite nature (cf. Figure~\ref{fig:VariableGadget}(a)). A matching that saturates all of the  $v^i$-vertices (black) of this gadget  is of one of two types, illustrated in Figure~\ref{fig:VariableGadget}(b) and (c)), which we associate with the two possible truth assignments to $V^i$. 


\begin{figure}[htbp]
\centerline{\includegraphics[scale=0.7]{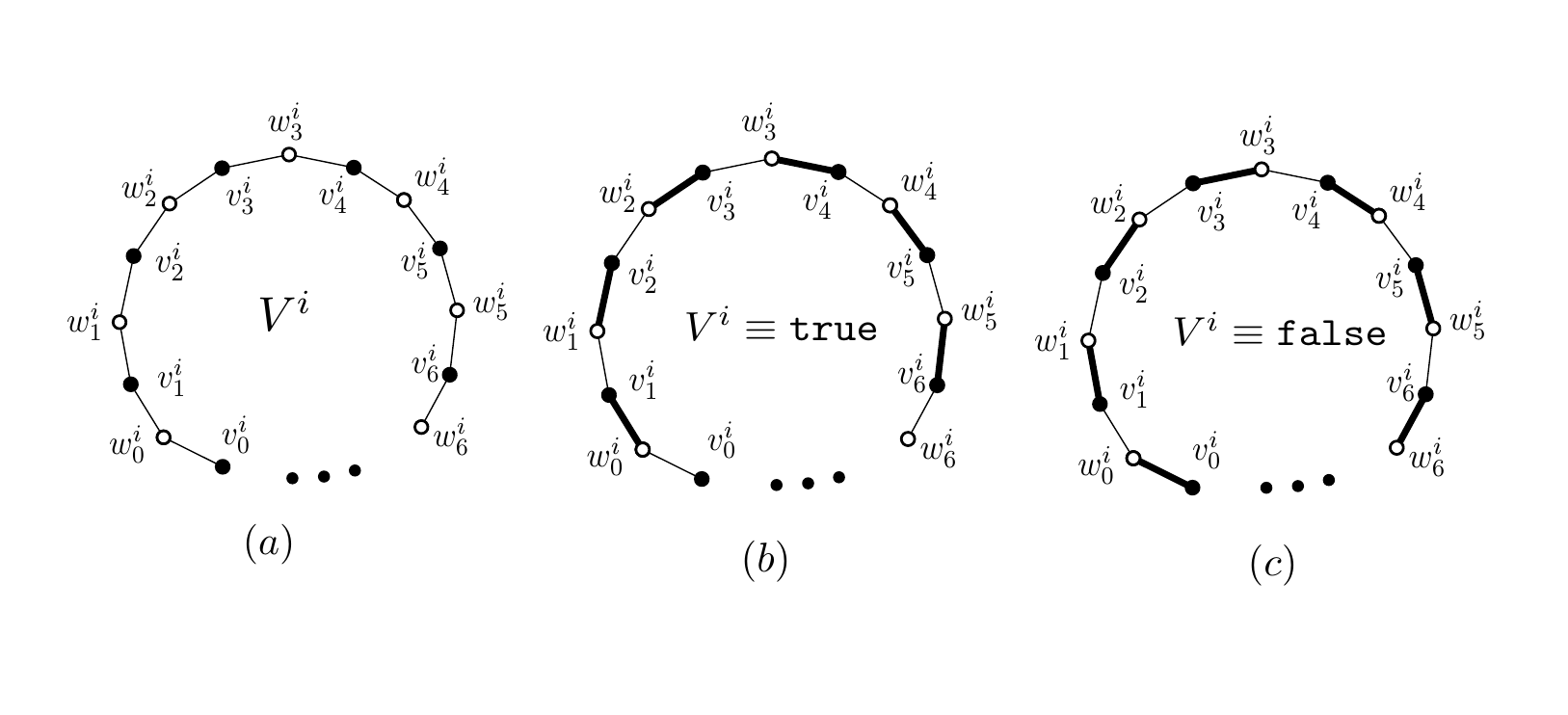}}
\vspace*{-1.2cm}
\caption{VariableGadget}
\label{fig:VariableGadget}
\end{figure}

We associate with each clause $D^j$ a \emph{clause gadget} consisting of $10$ vertices, with subscripted labels $p^j$, $q^j$, $r^j$ and $s^j$ (cf. Figure~\ref{fig:ClauseGadget}(a)). It is straightforward to confirm that any matching that saturates all of the $r^j$ and $q^j$-vertices (black) must use exactly one of the three $p^j q^j$-edges, illustrated in Figure~\ref{fig:ClauseGadget}(b) (c) and (d)). We refer to the $p^j q^j$-edges as \emph{portals} of the clause gadget, since their endpoints are the only points of connection with other parts of the full construction.

\begin{figure}[htbp]
\centerline{\includegraphics[scale=0.7]{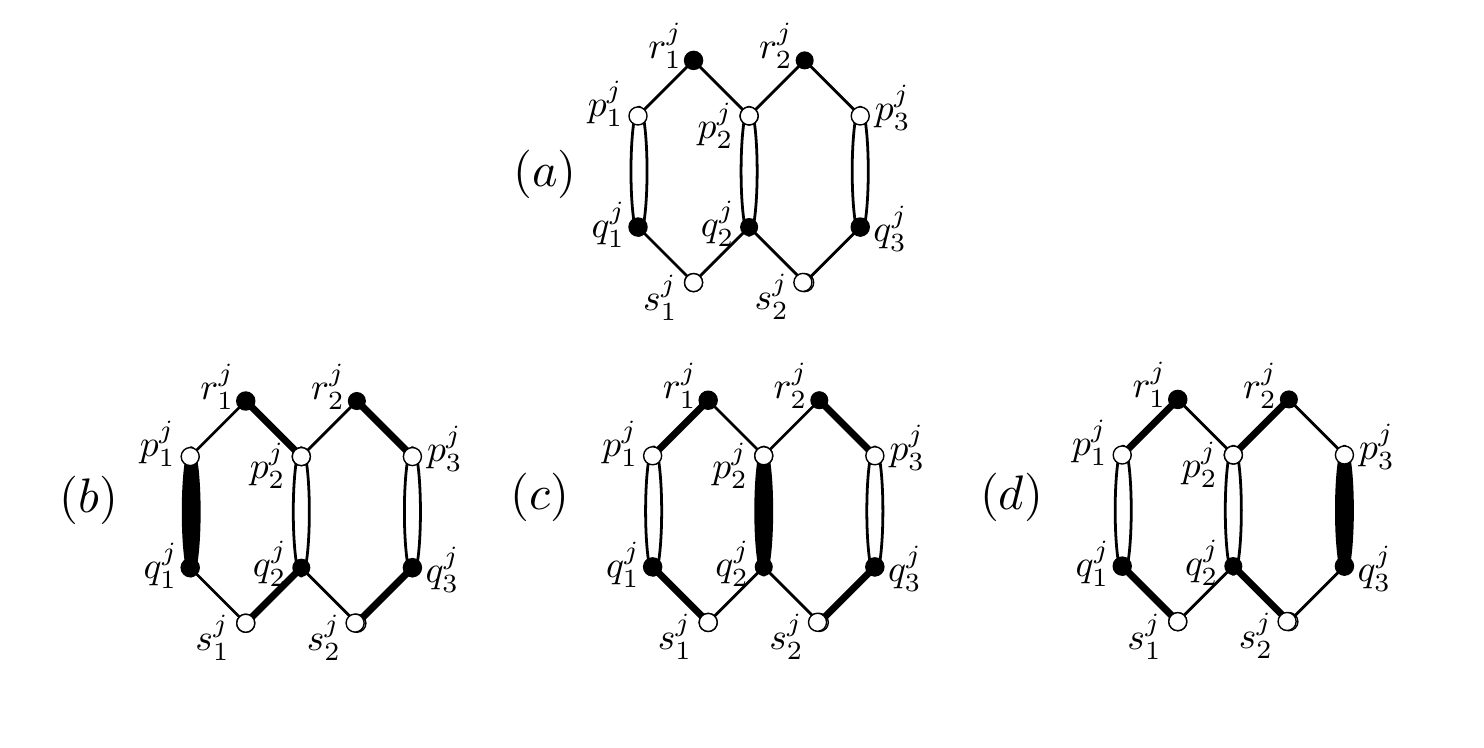}}
\caption{ClauseGadget}
\label{fig:ClauseGadget}
\end{figure}

We complete the construction by adding edges from vertex gadgets to appropriate clause gadget portals. Specifically, (i) if the $k$-th literal in clause $D^j$ is $V^i$, then we add edges from $v^i_{2j}$ to $p^j_k$ and $q^j_k$ to $w^i_{2j}$ (cf. Figure~\ref{fig:ConnectorGadgets}(a)) and (ii)  if the $k$-th literal in clause $D^j$ is $\overline{V^i}$, then we add edges from $v^i_{2j}$ to $p^j_k$ and $q^j_k$ to $w^i_{2j-1}$ (cf. Figure~\ref{fig:ConnectorGadgets}(b)). These \emph{connector} edges, shown dashed in Figures~\ref{fig:ConnectorGadgets}(a) and (b), are forbidden in any matching satisfying the constraints set out above, by the inclusion, for each such edge, of a pair of additional vertices and associated bridging path, as illustrated in Figure~\ref{fig:ConnectorGadgets}(c).
(Observe that since the graph has the same number of black and white vertices, a matching that saturates all of the black vertices must also saturate all of the white vertices. Thus, for each connector edge, the middle edge of its bridging path is forced to belong to the matching; otherwise, the end edges of the bridging path must both be chosen, resulting in a clash.)

It follows that  if the $k$-th literal in clause $D^j$ is $V^i$, and the edge $p^j_k q^j_k$ belongs to the constrained matching then edge $v^i_{2j} w^i_{2j}$ cannot belong. Similarly,  if the $k$-th literal in clause $D^j$ is $\overline{V^i}$, and the edge $p^j_k q^j_k$ belongs to the constrained matching then edge $v^i_{2j} w^i_{2j-1}$ cannot belong.

\begin{figure}[htbp]
\centerline{\includegraphics[scale=0.8]{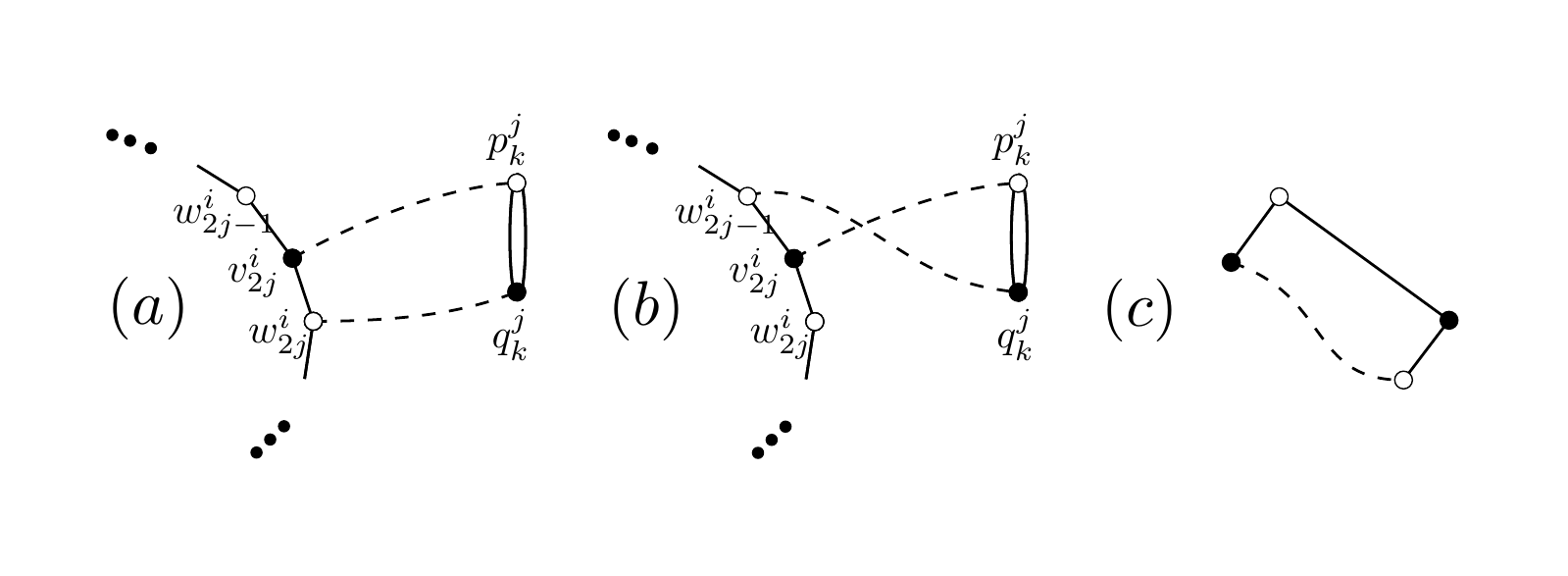}}
\vspace*{-1.2cm}
\caption{Connector Gadgets}
\label{fig:ConnectorGadgets}
\end{figure}


To complete the proof it remains to argue that the resulting graph $B_{\cal D}$ admits a matching $M$ such that (i) $M$ saturates all of the black vertices, and (ii) no two edges of $M$ are part of a $4$-cycle in $B_{\cal D}$, if and only if the instance $\cal D$ is satisfiable. Suppose first that $B_{\cal D}$ admits such a matching $M$. Since none of the connector edges are included in $M$, it follows (as argued above) that in every vertex gadget the black vertices are saturated in one of the two ways illustrated in Figure~\ref{fig:VariableGadget}(b) and \ref{fig:VariableGadget}(c)). Similarly, in every clause gadget, the black vertices are saturated in one of the three ways illustrated in Figure~\ref{fig:ClauseGadget}(b), \ref{fig:ClauseGadget}(c) and \ref{fig:ClauseGadget}(d)). Suppose that the portal edge $p^j_k q^j_k$ of the gadget associated with clause $D^j$ belongs to the matching $M$. Then, by our choice of connector edges, if the $k$-th literal in clause $D_j$ is $V^i$, it must be that edge $v^i_{2j} w^i_{2j}$ does not belong to $M$, that is the matching on the variable gadget associated with $V^i$ has the associated truth assignment $\tt true$. Similarly, if the $k$-th literal in clause $D_j$ is $\overline{V^i}$, it must be that edge $v^i_{2j} w^i_{2j-1}$ does not belong to $M$, that is the matching on the variable gadget associated with $V^i$ has the associated truth assignment $\tt false$. It follows that the truth assignment to the variables in $\cal V$, associated with the matchings induced on the vertex gadgets, satisfies all of the clauses in $\cal D$.

On the other hand, suppose that $\cal D$ is satisfiable, that is there is an assignment of truth values to the variables in $\cal V$ that satisfies all of the clauses in $\cal D$. Then, if we (i) choose the matching on the vertex gadget associated with $V^i$ to be the one corresponding to its truth assignment, and (ii) choose any matching on the clause gadget associated with clause $D^j$ including a portal edge associated with one of the satisfied literals in $D^j$, and (iii) choose all of the edges added to prevent the choice of connector edges, it is straightforward to confirm that the chosen edges form a matching $M$ in $B_{\cal D}$ such that (i) $M$ saturates all of the black vertices, and (ii) no two edges of $M$ are part of a $4$-cycle in $B_{\cal D}$.
\end{proof}

\section{Complexity of Degree-bounded Instances of Non-clashing Bipartite Matching}\label{app:tighterbounds}

The reduction described in the proof of Theorem~\ref{thm:NCmatching} produces a bipartite graph whose vertices have degree at most five. 
(Degree five is attained for the vertices $p^j_2$ and $q^j_2$ of the clause gadgets, both of which have three incident edges within the gadget and two from a bridged connector.) 
It is natural to ask if the hardness result continues to hold for bipartite graphs all of whose vertices have degree strictly less than five. 
In the next subsection we describe 
a fairly simple modification of both our clause and connector structures that allows us to reduce the maximum degree to three. Following that, we show that 
if the maximum degree among vertices in either part of a given bipartite graph is reduced to two there is a polynomial time algorithm to decide if the graph admits a non-clashing matching.

\subsection{A modified reduction with maximum degree three}

We begin by describing a new clause gadget, illustrated in Figure~\ref{fig:NewClauseGadget}(a), with the same $p$-$q$ portal structure as before but with the additional property that all $p$ and $q$ vertices have degree two. It is straightforward to confirm that, up to symmetry, the matching illustrated in Figure~\ref{fig:NewClauseGadget}(b) is the only matching that saturates all of the vertices using only edges internal to the gadget.

\begin{figure}[htbp]
\centerline{\includegraphics[scale=0.7]{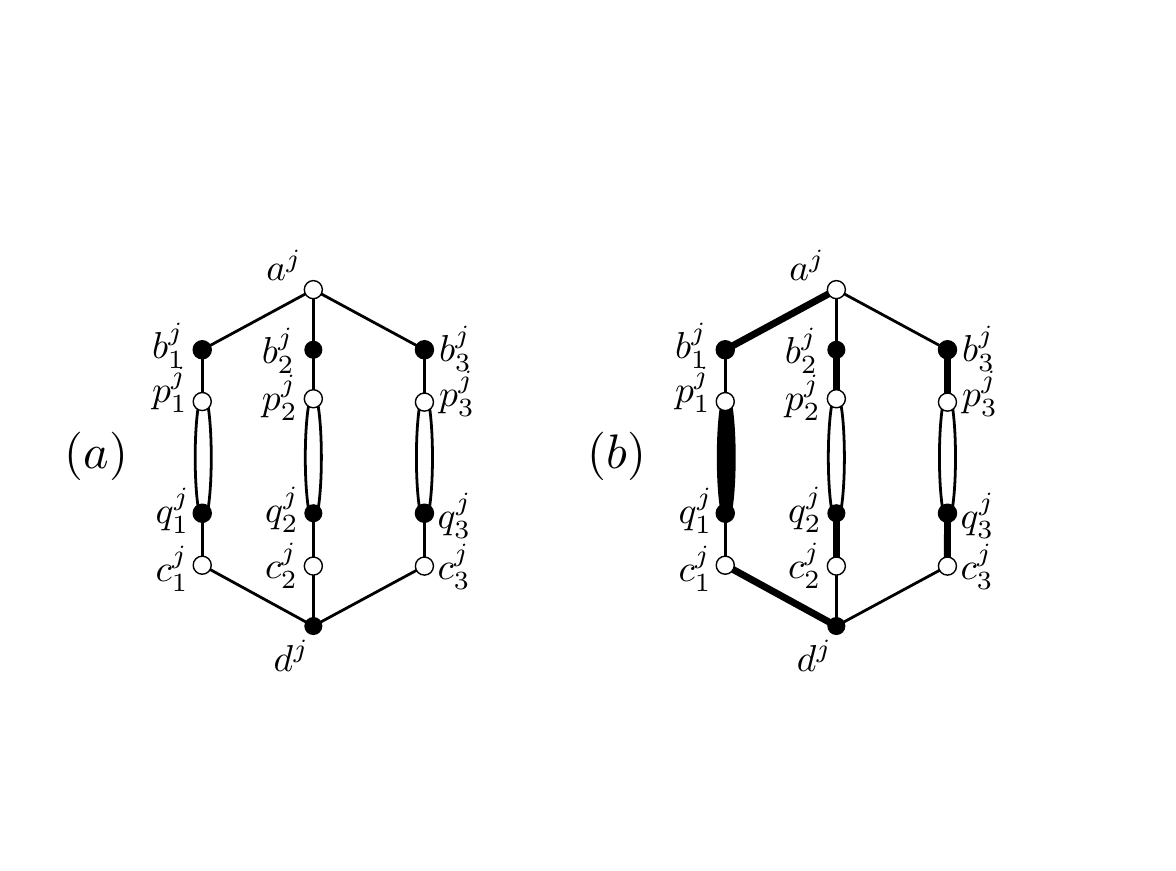}}
\caption{NewClauseGadget}
\label{fig:NewClauseGadget}
\end{figure}

Next we describe a somewhat more complicated connector structure that is used to link vertices in the variable gadgets with portal vertices of the new clause gadget. Schematically, as illustrated in Figures~\ref{fig:NewConnectorGadgets}(a) and (b), the connector structure plays exactly the same role as its counterpart (pair of bridged edges) in the earlier construction. The new connector structure, illustrated in Figures~\ref{fig:NewConnectorGadgets}(c), also contains edges, dashed as before, that cannot be part of any perfect non-clashing matching. Their role, as before, is simply to constrain the choice of other edges (in any perfect non-clashing matching).

\begin{figure}
\centerline{\includegraphics[scale=0.5]{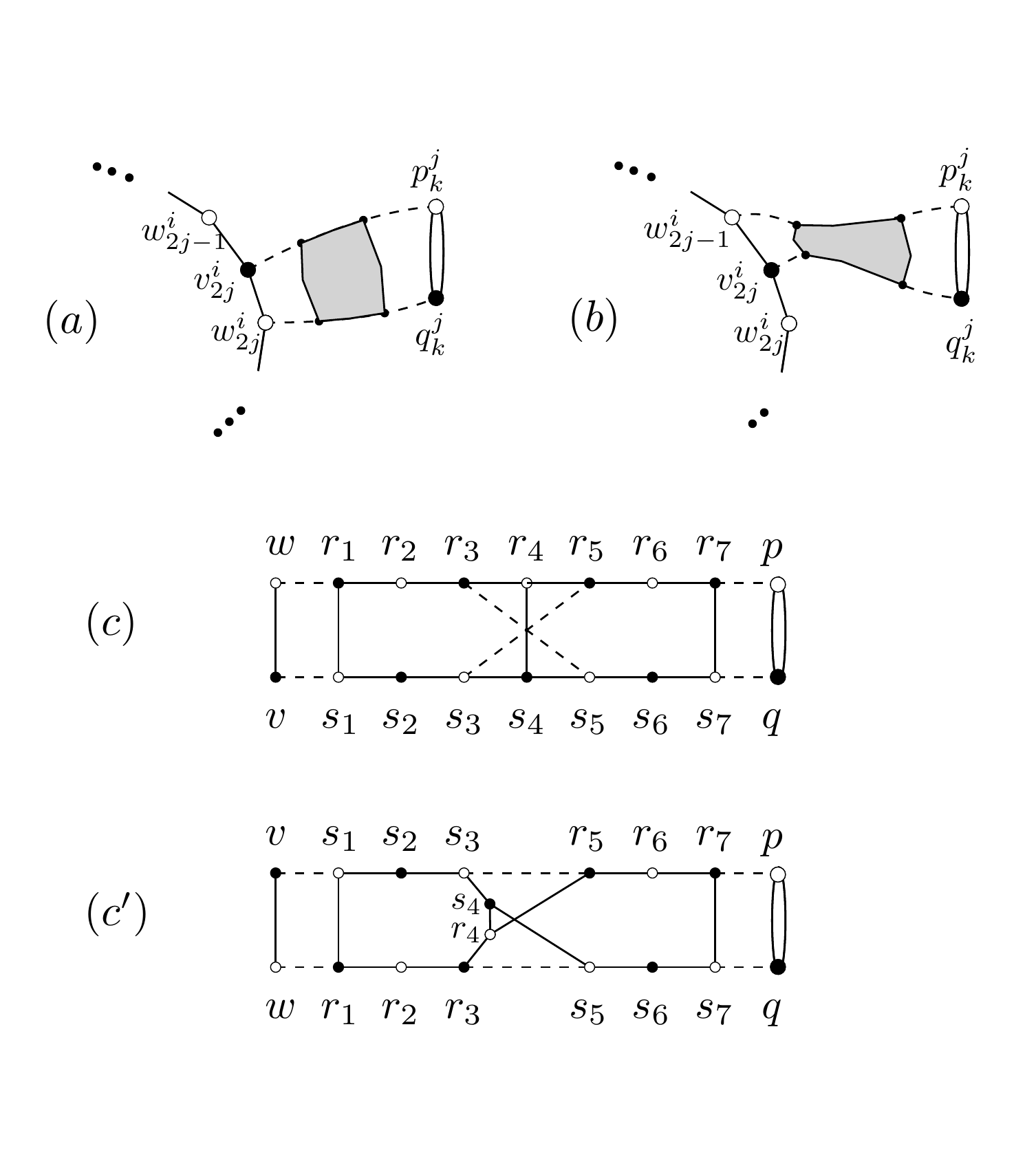}}
\caption{NewConnectorGadgets}
\label{fig:NewConnectorGadgets}
\end{figure}

It is easiest to argue first that neither of the dashed diagonals can be used. If both are used then edge $r_4s_4$ must also be used, creating a clash. On the other hand if just one, say $r_3s_5$ is used, then either $r_4s_4$ must also be used or both $r_4r_5$ and $s_3s_4$ must be used, creating a clash in either case. 

By parity, an even number of the horizontal dashed edges are used in any perfect matching. Since it is impossible to choose both $w r_1$ and $v s_1$ (or both $r_7 p$ and $s_7 q$) in a non-clashing matching, it suffices to rule out the case where exactly one of $w r_1$ and $v s_1$ and exactly one of $r_7 p$ and $s_7 q$ belong to a perfect matching. Suppose $r_7 p$ (but not $s_7 q$) is chosen. Then the matching is forced to include $r_5 r_6$ and $s_6 s_7$ (in order to saturate $r_6$ and $s_7$). This in turn forces the choice of $r_3 r_4$ and $s_4 s_5$ (in order to saturate $r_4$ and $s_5$), creating a clash. By symmetry, it follows that none of the horizontal dashed edges can be used in a perfect non-clashing matching.

It remains to argue that (i) if a non-clashing matching contains edge $pq$ then edge $vw$ cannot belong (and vice versa); (ii) there is a non-clashing matching of the connector gadget that contains edge $pq$ but leaves both $v$ and $w$ exposed (and vice versa); and (iii) there is a non-clashing matching of the connector gadget that  leaves all of  $v$, $w$, $p$ and $q$ exposed. For (i), we observe that, by chained forcing as above, the inclusion of $pq$ forces the inclusion of $r_1s_1$ (and, by symmetry, the inclusion of $vw$ forces the inclusion of $r_7s_7$). Properties (ii) and (iii) are illustrated in Figure~\ref{fig:ConnectorMatchings}.

\begin{figure}[htbp]
\centerline{\includegraphics[scale=0.6]{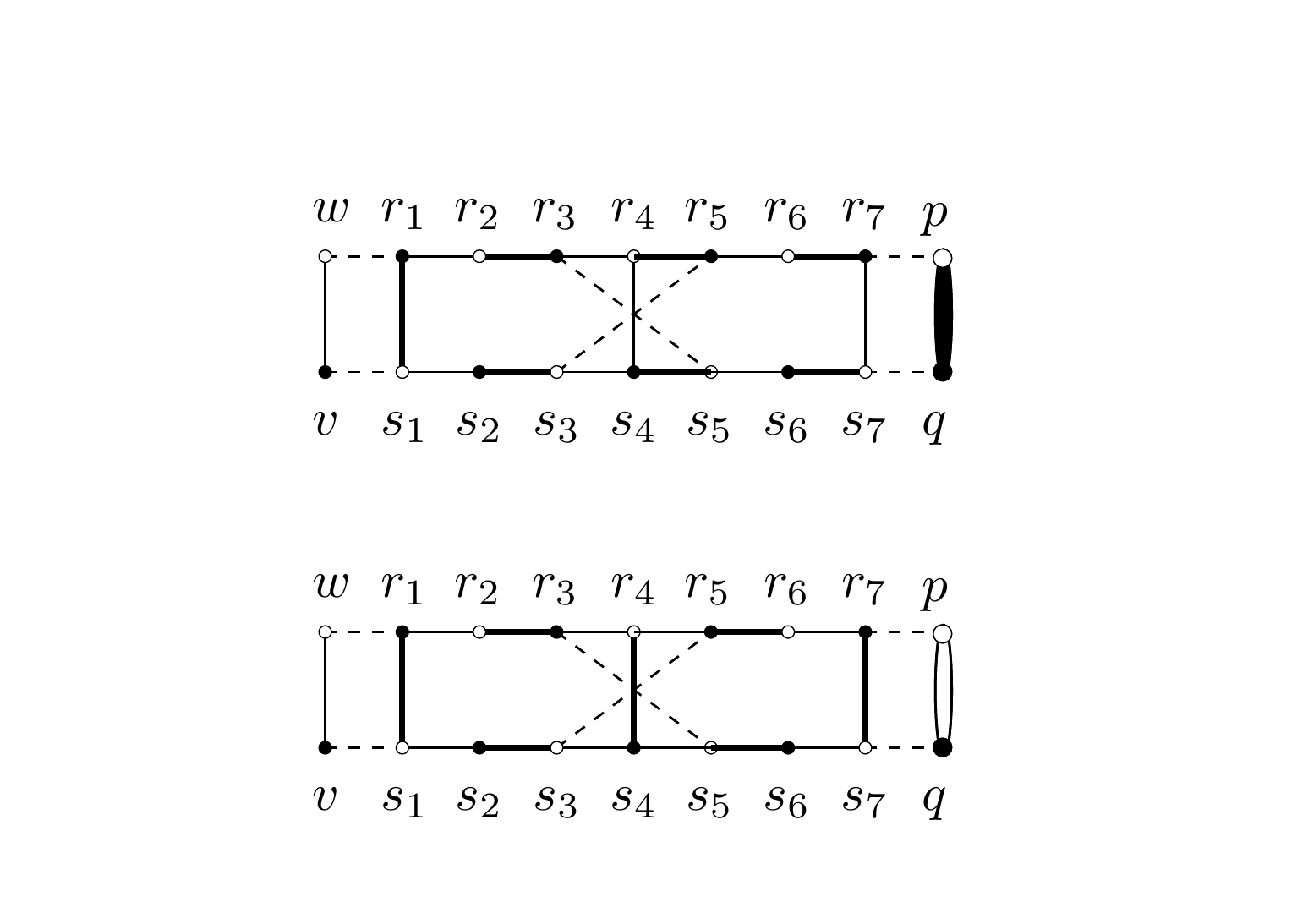}}
\caption{ConnectorMatchings}
\label{fig:ConnectorMatchings}
\end{figure}


\subsection{An efficient algorithm for Non-Clashing Bipartite Matching, when the maximum degree on either part is at most two}

Suppose we are given a bipartite graph $B$ whose vertices are either black or white, and all edges join a black vertex to a white vertex. We want to determine if  
$B$ admits a matching $M$ such that (i) $M$ saturates all of the black vertices, and (ii) no two edges of $M$ are part of a $4$-cycle in $B$.


Suppose further that the vertices on one of the two parts of $B$ all have degree at most two. We can assume, without loss of generality that they all have
degree exactly two, since edges with an endpoint of degree one can be (incrementally) included in a maximum matching $M$ without risk of being part of a 
$4$-cycle in $B$.  

We say that a pair of vertices in this degree-bounded part are \emph{twins} if they have the same adjacent vertices. We can assume that $B$ has no twins since (i) twins cannot both be saturated without producing a forbidden $4$-cycle, and therefore (ii) the existence of black twins immediately precludes a non-clashing matching, and (iii) any pair of white twins can be replaced by a single copy of the twinned vertex.

With this simplification it is easy to confirm that any matching that saturates the black vertices, the existence of which can be determined in polynomial time, must be non-clashing,

\end{document}